\documentclass[10pt,reqno]{article}

\usepackage{newtxtext,newtxmath} 
\usepackage{parskip}
\usepackage{microtype}
\usepackage{subcaption}
\usepackage{booktabs} 
\usepackage{natbib}
\usepackage{tikz}
\usepackage{authblk}
\usepackage{wrapfig}
\usepackage{sidecap} 

\usepackage{url}
\usepackage{enumerate} 
\usepackage{enumitem}
\usepackage{bbm}
\usepackage{amsfonts}
\usepackage{bm}

\usepackage{amsthm}

\usepackage{xcolor}
\definecolor{Amaranth}{rgb}{0.9, 0.17, 0.31}
\usepackage[colorlinks=true]{hyperref}%
\usepackage[capitalize, nameinlink, noabbrev]{cleveref}
\hypersetup{
citecolor=violet,
linkcolor=Amaranth,
urlcolor=cyan
}

\usepackage[margin=1in]{geometry}
\usepackage{mathtools}
\usepackage{relsize}

\newcommand\norm[1]{\left\|#1\right\|}
\usepackage{dsfont}
\usepackage{mathrsfs}
\usepackage{comment}
\providecommand{\cW}{\mathcal{W}}

\providecommand{\cF}{\mathcal{F}}
\providecommand{\cL}{\mathcal{L}}

\providecommand{\cB}{\mathcal{B}}
\providecommand{\cH}{\mathcal{H}}

\providecommand{\cD}{\mathcal{D}}
\providecommand{\cI}{\mathcal{I}}

\providecommand{\cA}{\mathcal{A}}
\providecommand{\cP}{\mathcal{P}}
\providecommand{\cG}{\mathcal{G}}

\providecommand{\cE}{\mathcal{E}}

\providecommand{\cC}{\mathcal{C}}

\def\1{\mathbbm{1}}

\newcommand{\f}{\mathbf{f}}

\newcommand{\bias}{\mathbf{b}}
\newcommand{\weights}{\mathbf{W}}

\usepackage{thm-restate}
\declaretheorem[name=Theorem]{theorem}
\declaretheorem[name=Lemma]{lemma}
\declaretheorem[name=Definition]{definition}
\declaretheorem[name=Remark]{remark}

\usepackage{tcolorbox}

\newcommand{\abs}[1]{\left|#1\right|}
\newcounter{RonCounter}

\newcounter{LeviCounter}

\usepackage{times}

\title{Neural Networks With Dense Computational Graphs \\ Are Not Universal Approximators}
\author[1]{Levi Rauchwerger}
\author[2,3]{Stefanie Jegelka}
\author[4]{Ron Levie}

\affil[1]{Princeton University, Dept of CS}
\affil[2]{MIT, Dept of EECS and CSAIL}
\affil[3]{TUM, School of CIT, MCML, MDSI}
\affil[4]{Technion – IIT, Faculty of Mathematics}

\date{ }
\begin{document}

\maketitle
\begin{abstract}
We investigate the approximation capabilities of dense neural networks. While universal approximation theorems establish that sufficiently large  architectures can approximate arbitrary continuous functions when there are no restrictions on the weights, we show that dense neural networks do not possess this universality. Our argument is based on a model compression approach, combining the weak regularity lemma with an interpretation of feedforward networks as message passing graph neural networks. We consider ReLU neural networks subject to natural constraints on weights and input and output dimensions, which model a notion of dense connectivity. Within this setting, we demonstrate the existence of Lipschitz continuous functions that cannot be approximated by such networks. This highlights intrinsic limitations of neural networks with dense layers and motivates the use of sparse connectivity as a necessary ingredient for achieving true universality.
\end{abstract}

\section{Introduction}



Deep neural networks are central to modern machine learning, driving advances in vision, language, and scientific applications~\citep{lecun2015deep}. Recently, deep networks have grown increasingly large, enhancing their performance while increasing computational costs~\citep{Thompson2020}, which motivates a focus on efficiency in both training and inference~\citep{Menghani2023}. 
One of the most prominent approaches to improving efficiency is the use of sparse neural networks. Sparse models can match or surpass dense ones in accuracy, using less energy, memory, and training time~\citep{mocanu2018,sparse_training_algo,sparsitySurvey,Peste2023}.
These properties have made sparsification methods such as pruning~\citep{lecun1990optimal,pruning1,han2015learning,Lee2019pruning} key in the use of deep networks on limited-resource devices and for lowering inference cost in large-scale systems \citep{modelcompressionforlowpowerapplications}. 
Strikingly, even pruning 80–90\% of parameters often leaves generalization (measured by test error) unaffected~\citep{Frankle2019lottaryHypothesis}. Beyond efficiency, sparsity can enhance optimization at certain levels, act as a built-in regularizer, and reduce sensitivity to noisy data \citep{sparse_mitigate_overfitting2022,sparse_mitigate_overfitting2022_2}. 

These empirical evidences raise the possibility that denseness itself may impose intrinsic limitations. While neural networks are known to approximate broad function classes~\citep{Yarotsky2017approximations}, it remains unclear whether such guarantees persist under dense connectivity constraints. This motivates the following central question: \textbf{Do dense deep neural networks have inferior expressivity?}

Our work focuses on the $L^\infty$-approximation error for $1$-Lipschitz continuous functions bounded by $1$ on a $d$-dimensional domain. Within this setting, the approximation capabilities of neural networks have been extensively studied since the development of classical universal approximation theorems \citep{Cybenko1989, HORNIK1991}. A central result is that any Lipschitz continuous (or more generally continuous) function on a $d$-dimensional compact domain can be approximated to arbitrary accuracy $\epsilon>0$ by a sufficiently large ReLU network~\citep{Mhaskar1996approximation,Pinkus1999approximation}. \cite{Yarotsky2017approximations} gave explicit error bounds for approximating functions in Sobolev spaces of given smoothness using deep and shallow ReLU networks. Notably, \cite{Yarotsky2017approximations} provides rigorous upper and lower bounds on the number of weights required for a given approximation error. \citet{Huster2018,Anil2019} show that ReLU networks whose weight norms are bounded by $1$ fail to approximate or represent the absolute value function. In contrast, our denseness constraint allows the weight bound and with it the per-layer Lipschitz constant to be an arbitrarily large constant. Moreover, we prove that a large set of Lipschitz functions cannot be approximated by densely connected MLPs.

In contrast to prior work that studies the expressivity of general neural networks, we specifically investigate the expressivity limits of dense networks. We consider dense networks whose weight matrices are uniformly bounded, with both entrywise $\ell_\infty$ control and normalized $\ell_1$ magnitude. These constraints enforce a strong form of denseness: nontrivial outputs require contributions from many neurons, so dense connectivity reflects genuine collective computation rather than a few dominating weights. By leveraging the graph structure of neural computation, we connect deep neural networks to the broader class of graph neural networks~\citep{Gil2017, xu2018powerful}, for which a rich theoretical framework has been developed
~\citep{signal23,signal25,expr23,DIDMs25,herbst2025higherorder,ICG,IBG} using dense graph limit theory~\citep{lovasz2012large}. Our analysis builds on the so-called \emph{weak regularity lemma}~\citep{frieze1999quick,lovas2007} to obtain an implicit compression result, showing that -- even as width increases (and depth is held fixed) -- the expressive power of dense networks saturates at a fixed resolution, thereby revealing intrinsic barriers that scaling alone cannot overcome. Our theoretical framework is fundamentally different from both compression-based generalization analyses \citep{arora2018stronger} and optimization-based perspectives such as in \citep{jacot2018neural}. Instead, we study intrinsic expressivity constraints of dense architectures, independent of training, generalization, or optimization dynamics, by leveraging the precise characterizations provided by \cite{Yarotsky2017approximations} and structural graph-theoretic constraints inherent to dense architectures. We are not aware of any prior work that uses compression to derive expressivity limits, particularly for dense neural networks.

\textbf{Contributions.}  We provide a theoretical account of a widely observed phenomenon: dense neural networks are fundamentally less expressive than their sparse counterpart. Our results can be summarized as follows. 
\begin{enumerate}[noitemsep, topsep=2pt, leftmargin=*]
    \item Classical results achieve universality by allowing width to increase with the desired approximation accuracy. In contrast, we show that, under some conditions on the input dimension, dense networks of any fixed depth cannot exploit width alone to achieve universality; their expressive power saturates, revealing an inherent limitation of dense architectures (\cref{thm:ExpressivityBoundsOne}).
    \item Our proof technique is not only novel but also offers a broadly applicable framework for analyzing neural networks. We identify two primary ways this technique can be adopted:
    \begin{itemize}
        \item \textbf{General Approximation Method:} The method can be generalized to any subclass of networks that admits a compression guarantee -- specifically, where there exists a small, bounded class that  approximates the original subclass.
        \item \textbf{Graph-Theoretic Perspective:} By treating the neural network as a computational graph, our approach utilizes tools from GNN theory -- specifically the regularity lemma -- to study architectural properties. This provides a blueprint for using GNN theory, and specifically regularity theorems, to analyze other aspects of neural networks.
    \end{itemize}
    \item As a byproduct, our analysis yields an implicit model compression scheme (\cref{thm:arch_approx}) showing that every dense ReLU network can be approximated by a bounded-size network.
\end{enumerate}
\section{Preliminaries}\label{sec:background}
  Here, we provide graph-theoretic background, including kernels as a dense analogue of graphs \citep{lovasz2012large}, introduce Euclidean and graph networks, propose a formal “strong” notion of dense networks, and recall definitions of network complexity and function approximation.
\paragraph{Graphs.}
For $n\in\mathbb{N}$, we write $[n]=:\{1,\ldots,n\}$. An \emph{attributed graph} (or \emph{graph-signal}) is a pair $(G,\mathbf{f})$, where $G$ is a weighted directed graph with node set $V=[n]$, edge set $E\subseteq V\times V$, and adjacency matrix $\mathbf{A}:=(A_{i,j})_{i,j\in [n]}$, and $\mathbf{f}:V\to \mathbb{R}$ is a \emph{node feature vector} (\emph{signal}). The signal assigns to each node $v \in V$ an attribute $\f(v) \in \mathbb{R}$. Let $B>0$, define a $[-B,B]$-weighted graph as a graph with edge weights in $[-B,B]$.
\paragraph{Partitions.} Let $\mu$ denote the Lebesgue measure on $[0,1]$. A \emph{partition}\label{p:partition} of $[0,1]$ is a sequence $\mathcal{P}_n = \{P_1, \dots, P_n\}$ of disjoint measurable sets such that $\bigcup_{j=1}^n P_j = [0,1]$. The partition is an \emph{equipartition} if $\mu(P_i) = \mu(P_j)$ for all $i,j \in [n]$. By $\cI_n$, we denote the equipartition of $[0,1]$ into $n$ intervals. A partition $\mathcal{Q}_l$ is called a \emph{refinement} of  $\mathcal{P}_k$ if every $Q\in \mathcal{Q}_l$ is a subset of some $P\in \mathcal{P}_k$. The above definitions extend naturally to general intervals $[a,b]$ instead of $[0,1]$.
\paragraph{Kernels.} In our analysis, we compare ``large'' neural networks to ``small'' neural networks  that compress them. To define a distance between the computational graphs of networks of different sizes, we embed graphs of all sizes into the space of kernels, where there is linear structure and a natural norm. Kernels generalize the idea of $[-1,1]$-weighted adjacency matrices to the continuous set of nodes $[0,1]$. An \emph{attributed kernel} (or \emph{kernel-signal}) is a pair $(K,f)$, where $K$ is a \emph{kernel}, i.e. a measurable function $K:[0,1]^2 \to [-1,1]$, and $f:[0,1]\to \mathbb{R}$ is a measurable, integrable function called \emph{feature vector} (or \emph{signal}).  By $\mathds{1}_S$, we denote the indicator function of a set $S$. 
We call a kernel $K$ a \emph{step kernel} w.r.t. a partition $\cP_n:=\{P_i\}^n_{i=1}$ if 
$K(x,y):=\sum_{i,j\in [n]} c_{i,j}\mathds{1}_{P_{i}\times P_{j}}(x,y)$ for some choice of coefficients $\{c_{i,j}\in [-1,1]\}_{i,j\in [n]}$. Similarly, we call a signal $f$ a \emph{step signal} w.r.t. a partition $\cP_n:=\{P_i\}^n_{i=1}$ if 
$f(x):=\sum_{i\in [n]} c_{i}\mathds{1}_{P_{i}}(x)$ for some choice of coefficients $\{c_{i}\in \mathbb{R}\}_{i\in [n]}$.
\paragraph{Graph-Induced Kernels.} Any $[-B,B]$-weighted graph can be naturally associated with a kernel. Let $G$ be a $[-B,B]$-weighted graph with adjacency matrix $\mathbf{A}=\{A_{i,j}\}_{i,j\in[n]}$ and signal $\f$. Consider the equipartition $\cI_n=\{I_k\}_{k=1}^{n}$ of $[0,1]$ into $n$ intervals. The kernel $K_G$ induced by $G$ is defined by
\(K_{G}(x,y)=\sum_{i,j=1}^{n} \frac{A_{i,j}}{B} \mathds{1}_{I_i}(x) \mathds{1}_{I_j}(y),\) with the induced signal $f_{\mathbf{f}}(z) = \sum_{i=1}^n \f_i \, \mathds{1}_{I_i}(z)$. This construction allows embedding the space of all graphs of all sizes into the space of kernels,\footnote{More accurately, the mapping from graphs to kernels is not one-to-one, and introduces equivalence classes of graphs that induce the same kernel. We will see that two equivalent computational graphs define the same function.}  where approximation properties are more naturally formulated. 

\paragraph{Cut Norm.} 
The cut norm, a fundamental distance measure in graph theory, was introduced by \cite{frieze1999quick} and underpins the cut distance, which serves as the central notion of convergence in the theory of dense graph limits \citep{lovasz2012large}. The \emph{cut norm} of an integrable \( K: [0,1]^2\to \mathbb{R} \) and an integrable $\f:[0,1]\to\mathbb{R}$ are respectively defined as
\[
\norm{K}_{\square} =\sup_{U,S\subset [0,1] } \abs{\int_{U\times S} K(x,y)dx dy}\quad\text{
and}\quad\norm{f}_{\square} := \sup_{S\subseteq [0,1]}
  \bigg|\int_{S} f(x) dx \bigg|.
\]
A distance between graphs can be defined as the norm of the difference between their induced kernels.
\paragraph{Neural Networks.}
\emph{Neural networks (NNs)}, also known as \emph{multilayer perceptrons (MLPs)}, process vector-valued inputs in $\mathbb{R}^{d_0}$ using a collection of \emph{learnable parameters} (or \emph{weights}) $(\mathbf{W},\mathbf{b})$. These consist of a sequence of weight matrices
\(
\mathbf{W} = (W^{(\ell)})_{\ell=1}^L,\) where \(W^{(\ell)} \in \mathbb{R}^{d_\ell \times d_{\ell-1}},\) and a sequence of bias vectors \(\mathbf{b} = (b^{(\ell)})_{\ell=1}^L\) where \(b^{(\ell)}\in\mathbb{R}^{d_\ell}\). We call $L\in\mathbb{N}$, with $L\ge2$, the \emph{depth} of the network and refer to each $0\le \ell\le L$ as a \emph{layer}. We call $d_{\ell}\in\mathbb{N}$ the \emph{width} (or alternately the \emph{dimension}) of $\ell^{\mathrm{th}}$ layer and refer to each $i\in [d_\ell]$ as a \emph{channel}. Specifically, we call $d_{0}$ and $d_L$ the \emph{input} and \emph{output dimensions}, respectively, and refer to each $i\in [d_0]$ as an \emph{input channel} and each $i\in[d_L]$ as an \emph{output channel}. The tuple $(L, d_0,\dots,d_L)$ is called the \emph{network architecture}. The number of network parameters is $W:=\sum^L_{\ell=1}\bigl(d_\ell\cdot d_{\ell-1}+d_\ell\bigr)$. The \emph{network size} is $n:=\frac{L+2}{L-1}\sum_{\ell=1}^{L-1}d_{\ell}$, serving as an effective size parameter; the rationale is given in \cref{sec:mlpgraph}.

\paragraph{Forward Propagation.}
We work throughout with a normalized parameterization in which each layer is scaled by a global factor depending on the network architecture. Given parameters $(\mathbf{W},\mathbf{b})$, the associated depth-$L$ ReLU network
\(\Theta_{(\mathbf{W},\mathbf{b})} : \mathbb{R}^{d_0} \to \mathbb{R}^{d_L}\)
is defined as follows. For an input $\mathbf{x}\in\mathbb{R}^{d_0}$, set $\mathbf{h}^{(0)} := \mathbf{x}$. The activation of the first hidden layer in channel $i\in[d_1]$ is
\begin{equation}
\mathbf{h}^{(1)}_i
= \mathrm{ReLU}\!\left(
\frac{1}{d_0(L+2)}
\sum_{j=1}^{d_0} \bigl(W^{(1)}_{ij}\mathbf{h}^{(0)}_j + b^{(1)}_i\bigr)
\right).
\label{eq:MLP}
\end{equation}
For layers $\ell = 2,\dots,L-1$ and channel $i\in[d_\ell]$,
\begin{equation}
\mathbf{h}^{(\ell)}_i
= \mathrm{ReLU}\!\left(
\frac{1}{n}
\sum_{j=1}^{d_{\ell-1}}
\bigl(W^{(\ell)}_{ij}\mathbf{h}^{(\ell-1)}_j + b^{(\ell)}_i\bigr)
\right).
\label{eq:MLP2}
\end{equation}
The output is computed as in \cref{eq:MLP2} with $\ell=L$, without ReLU. 
The network output is $\mathbf{h}^{(L)}=(\mathbf{h}^{(L)}_i)_{i=1}^{d_L}$. The different normalization in the first layer reflects that $d_0$ and $d_L$ are fixed in our asymptotic regime, while the effective network size $n$ may grow arbitrarily. 
\paragraph{$B$-Strongly Dense Networks.}

When weights are unconstrained, the normalized parameterization (in \cref{eq:MLP,eq:MLP2}) does not limit generality, since any global scaling can be absorbed into the layer weights. However, under uniform weight bounds, it has structural consequences: it forces any nontrivial network -- any network that does not map all inputs to near-zero outputs -- to be \emph{dense}: meaningful feature propagation requires many neurons to be  active. 
We refer to this as a \emph{strong} notion of denseness to distinguish it from weaker forms where a network might have many non-zero weights that are nonetheless large enough to allow a few neurons to dominate the output. 
\begin{definition}[$B$-Strongly Dense Network]
Let $B>0$. A depth-$L$ ReLU network $\Theta_{(\mathbf{W},\mathbf{b})}$ is called a \emph{$B$-infinity normalized ($B$-IN) dense network}, or more concisely, a \emph{$B$-strongly dense network}, if its parameters satisfy \(
|W^{(\ell)}_{i,j}| \le B\) and \(|b^{(\ell)}_i| \le B,
\)
for all layers $\ell \in [L]$, where $B$ is a constant independent of the network size $n$.
\end{definition}
In such networks, each neuron contributes $O(1/n)$ to the pre-activation, so no finite subset of neurons can dominate as $n\to\infty$. This contrasts with unbounded networks, where sparse representations with a few large weights remain possible. The parameters and the architecture of such networks are referred to as \emph{dense parameters} and \emph{dense architecture} (with weight bound $B$), respectively. 

\paragraph{Choice of the Denseness Notion.} 
Our analysis requires a notion of density that excludes sparse graphs, and there
are a number of ways to model sparsity and density in graphs. The first is
asymptotic: a sequence of graphs $(G_n)_{n \in \mathbb{N}}$, each with $n$
vertices, is dense if there is a constant $0 < C \leq 1$ such that the number
$E_n$ of edges of $G_n$ is bounded from below by $C n^2$. Hence, it is not appropriate for
 setting which deals with a single graph. Another approach is from the point of view of data structures: a
graph is dense if it is represented by a full $n \times n$ adjacency matrix.
This definition is not appropriate for us either, as sparse graphs are then a
special case of dense graphs, whereas we exactly want to study graphs that
cannot be sparse. We hence choose to define dense graphs as those that have a dense computational graph. These are defined in the sense of the
weak regularity lemma \citep{frieze1999quick} (see \cref{sec:compression} for more information). This gives a non-asymptotic
definition, which fulfills the requirement of our setting.
\paragraph{Graph and Kernel Neural Networks.}
\textit{Message passing graph neural networks (MPNNs)} form a class of neural networks designed to process graph-structured data  by iteratively updating node embeddings through the exchange of messages between nodes \citep{merkwirth05,gori05,xu2018powerful, Gil2017}. \cite{signal23, signal25} extend this framework canonically to operate on kernels. Since our construction is directly derived for kernels, we adopt their definitions.  Unlike general trainable MPNNs, we consider the following special predefined network. 

\paragraph{$B$-IR-MPNN.}
Let $B>0$ and $L\in\mathbb{N}$. We define the $B$-amplified integral-ReLU MPNN ($B$-IR-MPNN) $\Phi_{B,L}$, with $L$ layers, as the following mapping from attributed kernels to node features.
Given $(K,f)$, define  $\Phi_{B,L}(K,f) := f^{(L)}$, where with $f^{(0)}(x) := f(x)$ and for $\ell \in[L-1]$:
\begin{equation}\label{eq:mpnns}
f^{(\ell)}(x) := 
B\cdot{\mathrm{ReLU}}\!\left(\int_{[0,1]} K(x,y)f^{(\ell-1)}(y)dy\right).
\end{equation}
The output is computed as in \cref{eq:mpnns} with $\ell = L$, without ReLU. This definition is the canonical extension of MPNNs from graphs to kernels. For the definition of MPNNs on graphs and the equivalency to MPNNs on kernels  see \cref{app:graph_ect}.

\paragraph{Functions and Approximation.}
Our work focuses on the limitations of approximating Lipschitz continuous functions $[0,1]^d\to\mathbb{R}$ via $B$-strongly dense ReLU networks. Let ${\mathrm{C}}(d,d')$ denote the space of all continuous functions mapping $[0,1]^d$ to $\mathbb{R}^{d'}$.  The \emph{approximation error} between two functions \(f\) and \(g\) is
\(
\| f - g \|_{\infty} 
= \mathrm{ess}\sup_{\mathbf{x} \in [0,1]^d} \norm{ f(\mathbf{x}) - g(\mathbf{x}) }_\infty.
\) Here, the outer \(\|\cdot\|_\infty\) denotes the supremum over the input domain \([0,1]^d\), while the inner \(\|\cdot\|_\infty\) denotes the standard \(\ell_\infty\) norm on the output space (i.e., the maximum absolute value over coordinates). A function $f : [0,1]^d \to \mathbb{R}^{d'}$ is called \emph{$L$-Lipschitz continuous} if there exists a constant $L \geq 0$ such that \(
\| f(\mathbf{x}) - f(\mathbf{y}) \|_{\infty} \leq L \, \| \mathbf{x} - \mathbf{y}\|_{\infty}\) for all $\mathbf{x}, \mathbf{y} \in [0,1]^d$. Denote the space of $1$-Lipschitz continuous functions bounded by 1 as \({\mathrm{Lip}}(d,d')\).

\paragraph{Universal Approximation.}
A standard way to define the expressive power of a class of neural networks is via universal approximation.

\begin{definition}
Let $\mathcal{NN}$ and $\mathcal{F}$ be two sets of measurable functions $\Omega\rightarrow\mathbb{R}^d$ defined on the measure space $\Omega$.
    We say that $\mathcal{NN}$ is a \emph{universal approximator} of $\mathcal{F}$ if for every $\epsilon>0$ and $f\in\mathcal{F}$ there exists $\Theta\in\mathcal{NN}$ such that
    \[\|f-\Theta\|_{\infty}:=\mathrm{ess}\sup_{x\in\Omega}\|f(x)-\Theta(x)\|_{\infty}<\epsilon.\]
\end{definition}
A classical result states that the class of neural networks with arbitrary width and predefined fixed depth is a universal approximator of $\mathrm{C}(d,d')$ \citep{LESHNO1993861,Mhaskar1996approximation, Pinkus1999approximation}.
In our work, we take $\Omega=[0,1]^{d_0}$, choose $\mathcal{NN}$ to be a class of strongly dense neural networks (to be defined in \cref{sec:mlpgraph}) and show that it is not a universal approximator of $\mathcal{F}=\mathrm{Lip}(d_0,d_L)$. 

\paragraph{Network Complexity and the VC Dimension.}  
A classical way to formalize the  complexity of a hypothesis class of neural networks is through its \emph{VC-dimension}. The VC-dimension, which quantifies the largest set of points in an input space $\mathcal{X}$ that a hypothesis class $\mathcal{H}$ (a class of functions from $\mathcal{X}$ to $\{0,1\}$) can shatter, i.e., realize all possible Boolean labelings. If $\left| \{(h(x_1), \ldots, h(x_m)) : h \in \mathcal{H} \} \right| = 2^m$, we say $\mathcal{H}$ shatters the set $\{x_1, \ldots, x_m\}$. The Vapnik–Chervonenkis dimension of $\mathcal{H}$, denoted $\mathrm{VCdim}(\mathcal{H})$, is the size of the largest shattered set or $\infty$, if there is no such maximal set. 

Note that our complexity measures are largely insensitive to the specific choice of piecewise linear activation function. In fact, \cite{Yarotsky2017approximations}, Proposition 1, shows that any such activation with finitely many breakpoints can be simulated by ReLU networks with only a constant-factor increase in the number of units and weights. Thus, focusing on ReLU does not entail any real loss of generality.
\section{Deep Neural Networks as MPNNs}\label{sec:mlpgraph}
Deep networks are often described algebraically -- through layers, weights, and nonlinearities. Yet there is another perspective: every computation carried out by a network implicitly constructs a graph. Each node represents a feature channel and its value an activation, while each edge encodes a single weight of a weight matrix. Any feedforward network can thus be implemented via message passing, which propagates information through a specially designed directed graph -- its \emph{computational graph}. 

%
Rather than explicitly defining a computational graph and then inducing a corresponding kernel, we directly specify the class of kernels that encode valid neural network computations. These kernels, which we call \emph{computational kernels}, are constructed so that, for each dense feedforward neural network, message passing with ReLU nonlinearities on the corresponding kernel exactly reproduces the forward propagation of the network. For completeness, we provide the equivalent graph-based definition in \cref{app:graph_ect}.  
%
All results in the current section are proved in \cref{app:comp_graph}.
\paragraph{Fixed-Width Networks.} To simplify the presentation and subsequent analysis, we restrict the analysis to networks with fixed hidden width, i.e., \(d_\ell = d\) for all \(\ell=1,\dots,L-1.\) In this case, the network size is \(n = (L+2)d\), the number of parameters is $(L-2)d^2+(d_0+d_L+L-1)d+d_L$, and the forward propagation rule becomes, for all $\ell=1,\ldots,L-1$ and $i\in[d_\ell]$,
\[
\mathbf{h}^{(\ell)}_i
=
\mathrm{ReLU}\!\left(
\frac{1}{d_{\ell-1}(L+2)}
\sum_{j=1}^{d_{\ell-1}} \bigl(W^{(\ell)}_{ij}\mathbf{h}^{(\ell-1)}_j + b^{(\ell)}_i\bigr)
\right).
\]
The output is computed by the same formula with $\ell=L$, without ReLU.
 Note that we can absorb the $(L+2)$ factor by choosing $B=(L+2)C$ for some $C>0$. For simplicity, we assume that $d$ is divisible by both $d_0$ and $d_L$; this assumption is not essential and can be removed at the cost of additional notation. The space of all such networks, with any value of $d$, is denoted by $\mathcal{NN}(L,d_0,d_L)$ and similarly, the space of all such dense networks with weights and biases bounded in $[-B,B]$ by $\mathcal{NN}_{\text{dense}}(B,L,d_0,d_L)$.

\paragraph{The Computational Kernel of a Neural Network.} 
We call a kernel $K:[0,1]^2\to[-1,1]$ a depth-$L$ \emph{computational kernel} with respect to the parameters $L,d_0,d_L\in\mathbb{N}$ and $B\in\mathbb{R}$ with $B\ge L+2$ if $K$ satisfies the following four conditions (for geometric interpretation, consult Figure \ref{figure}). Each condition encodes a specific structural property of fixed-width feedforward networks, including layerwise connectivity, normalization of weights, and the representation of bias terms.

\begin{itemize}[leftmargin=0.35cm]\itemsep-0.3em
\item \emph{Condition 1.}
There exists $n\in\mathbb{N}$ divisible by $M(L+2)$, where $M$ is the least common multiple of $d_0$ and $d_L$, such that $K$ is a step kernel with respect to the interval equipartition $\mathcal{I}_n$ of $[0,1]$ into $n$ intervals. 
\item[ ]
\emph{Notation:}  $\mathcal{I}_n$ is a refinement of the interval equipartition
\(
\mathcal{U}_{L+2}:=\{U^{(0)},\dots,U^{(L)},U^{(\mathrm{bias})}\}
\) of $[0,1]$ into $L+2$ intervals of length $1/(L+2)$,   
where each $U^{(\ell)}$, $\ell=0,\dots,L$, is called a \emph{layer}, and $U^{(\mathrm{bias})}$ is called a \emph{bias}. We call $\mathcal{U}_{L+2}$ the \emph{layer partition}. 
For each $\ell$, define \(
\mathcal{R}^{(\ell)}_d := \{ I \in \mathcal{I}_n \mid I \subseteq U^{(\ell)} \}.
\) Here, $\mathcal{R}^{(\ell)}_d$ consists of $d:=n/(L+2)$ intervals. The partition $\mathcal{R}^{(0)}_d$ is the refinement of the coarser interval equipartition
$\mathcal{C}^{\mathrm{in}}_{d_0}$ of the interval $U^{(0)}$ into $d_0$ intervals of length $\frac{1}{(L+2)d_0}$ called the \emph{input partition}. The partition $\mathcal{R}^{(L)}_d$ is a refinement of the coarser interval equipartition
$\mathcal{C}^{\mathrm{out}}_{d_L}$ of the interval $U^{(L)}$ into $d_L$ intervals of length $\frac{1}{(L+2)d_L}$ called the \emph{output partition}.

\item \emph{Condition 2.}
 The kernel satisfies $K(x,y)=K(x,y')$ whenever $y$ and $y'$ belong to the same interval of $\mathcal{C}^{\mathrm{in}}_{d_0}$. Moreover,  $K(x,y)=K(x',y)$ whenever $x$ and $x'$ belong to the same interval of $\mathcal{C}^{\mathrm{out}}_{d_L}$. In addition, $K(x,y)=K(x,y')$ whenever $y$ and $y'$ belong to $U^{(\mathrm{bias})}$.
\setlength{\aboverulesep}{0pt}\setlength{\belowrulesep}{0pt}
\item  \emph{Condition 3.}
The kernel satisfies \(K(x,y)=0\) whenever either
\(y\in U^{(\ell)}\) and \(x\notin U^{(\ell+1)}\) for some
\(\ell\in\{0,\dots,L-1\}\), or \(y\in U^{(\mathrm{bias})}\) and
\(x\in U^{(0)}\), or \(x\in U^{(\mathrm{bias})}\) and \(y\notin U^{(\mathrm{bias})}\), or \(y\in U^{(L)}\).
\item  \emph{Condition 4.}
For any $x,y\in U^{(\mathrm{bias})}$, $K(x,y)=(L+2)/B$.
\end{itemize}
We call $d_0$ the \emph{input dimension}, $d_L$ the \emph{output dimension}, and refer to $d=n/(L+2)$ as the \emph{hidden dimension} for $\ell=1,\dots,L-1$. We call $n$ the \emph{size} of the computational kernel. By $\mathcal{CK}(B,L,d_0,d_L)$, we denote the collection of all depth-$L$ computational kernels as defined above. 
\begin{remark}
\label{rem:inducedGeneral}
Note that for any computational kernel there is a dense neural network that induces it. Similarly, every computational input signal is induced by some input to the network.
\end{remark}

\begin{figure}[t]
\centering
\begin{tikzpicture}[scale=3]

\def\xBA{0.000}    
\def\xBB{0.025}    
\def\xBC{0.050}
\def\xBD{0.075}
\def\xBE{0.100}
\def\xBF{0.125}
\def\xBG{0.150}
\def\xBH{0.175} 

\def\xLA{0.200}
\def\xLB{0.225}
\def\xLC{0.250}
\def\xLD{0.275}
\def\xLE{0.300}
\def\xLF{0.325}
\def\xLG{0.350}
\def\xLH{0.375}

\def\xMA{0.400}
\def\xMB{0.425}
\def\xMC{0.450}
\def\xMD{0.475}
\def\xME{0.500}
\def\xMF{0.525}
\def\xMG{0.550}
\def\xMH{0.575}

\def\xNA{0.600}
\def\xNB{0.625}
\def\xNC{0.650}
\def\xND{0.675}
\def\xNE{0.700}
\def\xNF{0.725}
\def\xNG{0.750}
\def\xNH{0.775}

\def\xOA{0.800}  
\def\xOB{0.825}
\def\xOC{0.850}
\def\xOD{0.875}
\def\xOE{0.900}
\def\xOF{0.925}
\def\xOG{0.950}
\def\xOH{0.975}  


\def\yBA{\xBA}
\def\yBB{\xBB}
\def\yBC{\xBC}
\def\yBD{\xBD}
\def\yBE{\xBE}
\def\yBF{\xBF}
\def\yBG{\xBG}
\def\yBH{\xBH}

\def\yLA{\xLA}
\def\yLB{\xLB}
\def\yLC{\xLC}
\def\yLD{\xLD}
\def\yLE{\xLE}
\def\yLF{\xLF}
\def\yLG{\xLG}
\def\yLH{\xLH}

\def\yMA{\xMA}
\def\yMB{\xMB}
\def\yMC{\xMC}
\def\yMD{\xMD}
\def\yME{\xME}
\def\yMF{\xMF}
\def\yMG{\xMG}
\def\yMH{\xMH}

\def\yNA{\xNA}
\def\yNB{\xNB}
\def\yNC{\xNC}
\def\yND{\xND}
\def\yNE{\xNE}
\def\yNF{\xNF}
\def\yNG{\xNG}
\def\yNH{\xNH}

\def\yOA{\xOA}
\def\yOB{\xOB}
\def\yOC{\xOC}
\def\yOD{\xOD}
\def\yOE{\xOE}
\def\yOF{\xOF}
\def\yOG{\xOG}
\def\yOH{\xOH}

\definecolor{bias}{RGB}{200,200,200}
\definecolor{layer}{RGB}{90,120,220}


\foreach \tgt/\colval in {%
    \xLA/30, \xLB/70, \xLC/45, \xLD/85, \xLE/25, \xLF/60, \xLG/40, \xLH/75,
    \xMA/55, \xMB/35, \xMC/80, \xMD/50, \xME/65, \xMF/30, \xMG/90, \xMH/45,
    \xNA/40, \xNB/75, \xNC/55, \xND/30, \xNE/85, \xNF/50, \xNG/65, \xNH/35,
    \xOA/60, \xOB/40, \xOC/80, \xOD/55, \xOE/30, \xOF/70, \xOG/45, \xOH/85} {
    \foreach \src in {\yBA,\yBB,\yBC,\yBD,\yBE,\yBF,\yBG,\yBH} {
        \fill[white!\colval!black] (\tgt,\src) rectangle ++(0.025,0.025);
    }
}

\foreach \tgt in {\xBA,\xBB,\xBC,\xBD,\xBE,\xBF,\xBG,\xBH} {
    \foreach \src in {\yBA,\yBB,\yBC,\yBD,\yBE,\yBF,\yBG,\yBH} {
        \fill[gray] (\tgt,\src) rectangle ++(0.025,0.025);
    }
}

\foreach \tgt/\colval in {\xMA/30, \xMB/70, \xMC/45, \xMD/85,
                          \xME/25, \xMF/60, \xMG/40, \xMH/75} {
    \foreach \src in {\yLA,\yLB} {
        \fill[white!\colval!black] (\tgt,\src) rectangle ++(0.025,0.025);
    }
}
\foreach \tgt/\colval in {\xMA/55, \xMB/35, \xMC/80, \xMD/50,
                          \xME/65, \xMF/30, \xMG/90, \xMH/45} {
    \foreach \src in {\yLC,\yLD} {
        \fill[white!\colval!black] (\tgt,\src) rectangle ++(0.025,0.025);
    }
}
\foreach \tgt/\colval in {\xMA/40, \xMB/75, \xMC/55, \xMD/30,
                          \xME/85, \xMF/50, \xMG/65, \xMH/35} {
    \foreach \src in {\yLE,\yLF} {
        \fill[white!\colval!black] (\tgt,\src) rectangle ++(0.025,0.025);
    }
}
\foreach \tgt/\colval in {\xMA/60, \xMB/40, \xMC/80, \xMD/55,
                          \xME/30, \xMF/70, \xMG/45, \xMH/85} {
    \foreach \src in {\yLG,\yLH} {
        \fill[white!\colval!black] (\tgt,\src) rectangle ++(0.025,0.025);
    }
}

\foreach \src/\cA/\cB/\cC/\cD/\cE/\cF/\cG/\cH in {%
    \yMA/30/65/45/80/25/55/40/70,
    \yMB/55/35/85/50/65/30/90/45,
    \yMC/40/75/55/30/80/50/65/35,
    \yMD/60/40/80/55/30/70/45/85,
    \yME/35/65/50/85/40/75/55/30,
    \yMF/70/45/30/65/55/35/80/50,
    \yMG/45/85/65/35/75/50/30/60,
    \yMH/80/30/55/70/45/65/40/75} {
    \fill[white!\cA!black] (\xNA,\src) rectangle ++(0.025,0.025);
    \fill[white!\cB!black] (\xNB,\src) rectangle ++(0.025,0.025);
    \fill[white!\cC!black] (\xNC,\src) rectangle ++(0.025,0.025);
    \fill[white!\cD!black] (\xND,\src) rectangle ++(0.025,0.025);
    \fill[white!\cE!black] (\xNE,\src) rectangle ++(0.025,0.025);
    \fill[white!\cF!black] (\xNF,\src) rectangle ++(0.025,0.025);
    \fill[white!\cG!black] (\xNG,\src) rectangle ++(0.025,0.025);
    \fill[white!\cH!black] (\xNH,\src) rectangle ++(0.025,0.025);
}

\foreach \src/\colA/\colB in {%
    \yNA/30/60, \yNB/75/35, \yNC/45/80, \yND/55/25,
    \yNE/85/50, \yNF/40/70, \yNG/65/30, \yNH/50/85} {
    \foreach \tgt in {\xOA,\xOB,\xOC,\xOD} {
        \fill[white!\colA!black] (\tgt,\src) rectangle ++(0.025,0.025);
    }
    \foreach \tgt in {\xOE,\xOF,\xOG,\xOH} {
        \fill[white!\colB!black] (\tgt,\src) rectangle ++(0.025,0.025);
    }
}


\draw[thick] (0,0) rectangle (1,1);

\draw[thick] (\xBA,0) -- (\xBA,1);
\draw[thick] (\xBB,0) -- (\xBB,1);
\draw[thick] (\xBC,0) -- (\xBC,1);
\draw[thick] (\xBD,0) -- (\xBD,1);
\draw[thick] (\xBE,0) -- (\xBE,1);
\draw[thick] (\xBF,0) -- (\xBF,1);
\draw[thick] (\xBG,0) -- (\xBG,1);
\draw[thick] (\xBH,0) -- (\xBH,1);

\draw[thick] (\xLA,0) -- (\xLA,1);
\draw[thick] (\xLB,0) -- (\xLB,1);
\draw[thick] (\xLC,0) -- (\xLC,1);
\draw[thick] (\xLD,0) -- (\xLD,1);
\draw[thick] (\xLE,0) -- (\xLE,1);
\draw[thick] (\xLF,0) -- (\xLF,1);
\draw[thick] (\xLG,0) -- (\xLG,1);
\draw[thick] (\xLH,0) -- (\xLH,1);

\draw[thick] (\xMA,0) -- (\xMA,1);
\draw[thick] (\xMB,0) -- (\xMB,1);
\draw[thick] (\xMC,0) -- (\xMC,1);
\draw[thick] (\xMD,0) -- (\xMD,1);
\draw[thick] (\xME,0) -- (\xME,1);
\draw[thick] (\xMF,0) -- (\xMF,1);
\draw[thick] (\xMG,0) -- (\xMG,1);
\draw[thick] (\xMH,0) -- (\xMH,1);

\draw[thick] (\xNA,0) -- (\xNA,1);
\draw[thick] (\xNB,0) -- (\xNB,1);
\draw[thick] (\xNC,0) -- (\xNC,1);
\draw[thick] (\xND,0) -- (\xND,1);
\draw[thick] (\xNE,0) -- (\xNE,1);
\draw[thick] (\xNF,0) -- (\xNF,1);
\draw[thick] (\xNG,0) -- (\xNG,1);
\draw[thick] (\xNH,0) -- (\xNH,1);

\draw[thick] (\xOA,0) -- (\xOA,1);
\draw[thick] (\xOB,0) -- (\xOB,1);
\draw[thick] (\xOC,0) -- (\xOC,1);
\draw[thick] (\xOD,0) -- (\xOD,1);
\draw[thick] (\xOE,0) -- (\xOE,1);
\draw[thick] (\xOF,0) -- (\xOF,1);
\draw[thick] (\xOG,0) -- (\xOG,1);
\draw[thick] (\xOH,0) -- (\xOH,1);

\draw[thick] (0,\yBA) -- (1,\yBA);
\draw[thick] (0,\yBB) -- (1,\yBB);
\draw[thick] (0,\yBC) -- (1,\yBC);
\draw[thick] (0,\yBD) -- (1,\yBD);
\draw[thick] (0,\yBE) -- (1,\yBE);
\draw[thick] (0,\yBF) -- (1,\yBF);
\draw[thick] (0,\yBG) -- (1,\yBG);
\draw[thick] (0,\yBH) -- (1,\yBH);

\draw[thick] (0,\yLA) -- (1,\yLA);
\draw[thick] (0,\yLB) -- (1,\yLB);
\draw[thick] (0,\yLC) -- (1,\yLC);
\draw[thick] (0,\yLD) -- (1,\yLD);
\draw[thick] (0,\yLE) -- (1,\yLE);
\draw[thick] (0,\yLF) -- (1,\yLF);
\draw[thick] (0,\yLG) -- (1,\yLG);
\draw[thick] (0,\yLH) -- (1,\yLH);

\draw[thick] (0,\yMA) -- (1,\yMA);
\draw[thick] (0,\yMB) -- (1,\yMB);
\draw[thick] (0,\yMC) -- (1,\yMC);
\draw[thick] (0,\yMD) -- (1,\yMD);
\draw[thick] (0,\yME) -- (1,\yME);
\draw[thick] (0,\yMF) -- (1,\yMF);
\draw[thick] (0,\yMG) -- (1,\yMG);
\draw[thick] (0,\yMH) -- (1,\yMH);

\draw[thick] (0,\yNA) -- (1,\yNA);
\draw[thick] (0,\yNB) -- (1,\yNB);
\draw[thick] (0,\yNC) -- (1,\yNC);
\draw[thick] (0,\yND) -- (1,\yND);
\draw[thick] (0,\yNE) -- (1,\yNE);
\draw[thick] (0,\yNF) -- (1,\yNF);
\draw[thick] (0,\yNG) -- (1,\yNG);
\draw[thick] (0,\yNH) -- (1,\yNH);

\draw[thick] (0,\yOA) -- (1,\yOA);
\draw[thick] (0,\yOB) -- (1,\yOB);
\draw[thick] (0,\yOC) -- (1,\yOC);
\draw[thick] (0,\yOD) -- (1,\yOD);
\draw[thick] (0,\yOE) -- (1,\yOE);
\draw[thick] (0,\yOF) -- (1,\yOF);
\draw[thick] (0,\yOG) -- (1,\yOG);
\draw[thick] (0,\yOH) -- (1,\yOH);

\node[above, font=\small] at (0.0875, 1.02) {bias};
\node[above, font=\small] at (0.2875, 1.02) {$\ell_0$};
\node[above, font=\small] at (0.4875, 1.02) {$\ell_1$};
\node[above, font=\small] at (0.6875, 1.02) {$\ell_2$};
\node[above, font=\small] at (0.8875, 1.02) {$\ell_3$};

\node[left, font=\small] at (-0.02, 0.0875) {bias};
\node[left, font=\small] at (-0.02, 0.2875) {$\ell_0$};
\node[left, font=\small] at (-0.02, 0.4875) {$\ell_1$};
\node[left, font=\small] at (-0.02, 0.6875) {$\ell_2$};
\node[left, font=\small] at (-0.02, 0.8875) {$\ell_3$};
\node[right, font=\normalsize] at (1.02, 0.5) {$y$ (source)};

\node[below, font=\normalsize] at (0.5, -0.05) {$x$ (target)};
\node[below, font=\normalsize] at (0.5, -0.25) {$K(x,y) \in \mathcal{CK}(L=3, d_0=4, d_L=2)$};

\def\sigX{3}
\def\cellHeight{0.025}

\foreach \i in {0,...,7} {
    \pgfmathsetmacro{\yPos}{\i * \cellHeight}
    \fill[black] (\sigX, \yPos) rectangle ++(0.05, \cellHeight);
}

\foreach \i in {8,9}   { \pgfmathsetmacro{\yPos}{\i * \cellHeight} \fill[white!30!black] (\sigX, \yPos) rectangle ++(0.05, \cellHeight); }
\foreach \i in {10,11} { \pgfmathsetmacro{\yPos}{\i * \cellHeight} \fill[white!75!black] (\sigX, \yPos) rectangle ++(0.05, \cellHeight); }
\foreach \i in {12,13} { \pgfmathsetmacro{\yPos}{\i * \cellHeight} \fill[white!50!black] (\sigX, \yPos) rectangle ++(0.05, \cellHeight); }
\foreach \i in {14,15} { \pgfmathsetmacro{\yPos}{\i * \cellHeight} \fill[white!90!black] (\sigX, \yPos) rectangle ++(0.05, \cellHeight); }

\foreach \i/\colval in {16/35, 17/70, 18/45, 19/85, 20/55, 21/30, 22/65, 23/40} {
    \pgfmathsetmacro{\yPos}{\i * \cellHeight}
    \fill[white!\colval!black] (\sigX, \yPos) rectangle ++(0.05, \cellHeight);
}

\foreach \i/\colval in {24/60, 25/40, 26/80, 27/55, 28/30, 29/70, 30/45, 31/85} {
    \pgfmathsetmacro{\yPos}{\i * \cellHeight}
    \fill[white!\colval!black] (\sigX, \yPos) rectangle ++(0.05, \cellHeight);
}

\foreach \i in {32,33,34,35} { \pgfmathsetmacro{\yPos}{\i * \cellHeight} \fill[white!40!black] (\sigX, \yPos) rectangle ++(0.05, \cellHeight); }
\foreach \i in {36,37,38,39} { \pgfmathsetmacro{\yPos}{\i * \cellHeight} \fill[white!75!black] (\sigX, \yPos) rectangle ++(0.05, \cellHeight); }

\draw[thick] (\sigX, 0) rectangle ++(0.05, 1);

\node[below, font=\normalsize] at (\sigX + 0.025, -0.05) {$x$ (input)};
\node[below, font=\normalsize] at (\sigX + 0.025, -0.25) {$f(x)\in\mathcal{CS}(L=3,d_0=4,d_L=2)$};

\end{tikzpicture}
\caption{A computational kernel-signal with $d=8$. The kernel $K$ is nonzero only on blocks linking consecutive layers and bias-to-layer blocks; visible stripes in the first/last layer blocks reflect constancy over input/output coarse intervals. The signal $f$ shares the same coarse-interval structure.}
\label{figure}
\end{figure}

\paragraph{Computational Signals.} While the kernel encodes weights, the signal encodes activations -- inputs and propagated values. Given parameters \(L,d_0,d_L\), a \emph{computational signal} is any step signal with respect to the partition $\cup_{\ell=1}^{L-1} \mathcal{R}_d^{(\ell)}\cup\mathcal{C}^{\mathrm{in}}_{d_0}\cup \mathcal{C}^{\mathrm{out}}_{d_L}\cup\{U^{(\mathrm{bias})}\}$. Denote the space of computational signals by  $\mathcal{CS}(L,d_0,d_L)$. 
We define a \emph{computational input signal} as any computational signal satisfying \(f(x)=1\) for all \(x\in U^{(\mathrm{bias})}\) and  \(f(x)=0\) for all \(x\notin U^{(0)}\cup U^{(\mathrm{bias})}\). 
The values of \(f\) on the intervals of \(\mathcal{C}^{\mathrm{in}}_{d_0}\) are interpreted as the input vector to the network, while the constant value on \(U^{(\mathrm{bias})}\) represents a bias signal that is present at every layer. \emph{Condition 4} in the definition of a computational kernel, i.e., \(K(x,y)=(L+2)/B\) for all \(x,y\in U^{(\mathrm{bias})}\), ensures that this bias signal is unchanged during propagation, i.e., under successive applications of the computational kernel on the signal. 

We call a pair in $\mathcal{CK}(B,L,d_0,d_L)\times \mathcal{CS}(L,d_0,d_L)$ a \emph{computational kernel-signal}. 
\paragraph{Computational Kernel Induced by Network Parameters.} Given a dense neural network, one can induce a computational kernel on  which the $B$-IR-MPNN implements the forward propagation of the neural network. 
Let $(\weights,\bias)$ be the dense parameters of the neural network $\Theta_{(\weights,\bias)}\in\mathcal{NN}_{\mathrm{dense}}(B,L,d_0,d_L)$, with hidden dimension $d$ divisible by $d_0$ and $d_L$. Denote $n=d(L+2)$. We define the computational kernel $K=K_{(\weights,\bias)}\in\mathcal{CK}(B,L,d_0,d_L)$ \emph{induced} by $\Theta_{(\weights,\bias)}$ as follows. For every $(x,y)\in[0,1]^2$:
\begin{itemize}[leftmargin=0.35cm]\itemsep-0.4em
    \item $K(x,y)=W^{(1)}_{i,j}/B$ if $y$ is in the $j$th interval of $\mathcal{C}^{\mathrm{in}}_{d_0}$ and $x$ is in the $i$th interval of $\mathcal{R}_d^{(1)}$. 
    \item   
    $K(x,y)=W^{(L)}_{i,j}/B$ if $y$ is in the $j$th interval of $\mathcal{R}_d^{(L-1)}$ and $x$ is in the $i$th interval of $\mathcal{C}^{\mathrm{out}}_{d_L}$. 
    \item   
    For $\ell=2,\ldots,L-1$, $K(x,y)=W^{(\ell)}_{i,j}/B$ if $y$ is in the $j$th interval of $\mathcal{R}_d^{(\ell-1)}$ and $x$ is in the $i$th interval of $\mathcal{R}_d^{(\ell)}$. 
    \item For $\ell=1,\ldots,L-1$, $K(x,y)=b^{(\ell)}_i/B$ if $y$ is in $U^{(\mathrm{bias})}$ and $x$ is in the $i$th interval of $\mathcal{R}_d^{(\ell)}$. Moreover, $K(x,y)=b^{(L)}_i/B$ if $y$ is in $U^{(\mathrm{bias})}$ and $x$ is in the $i$th interval of $\mathcal{C}^{\mathrm{out}}_{d_L}$.
    \item For any $x,y\in U^{(\mathrm{bias})}$, $K(x,y)=(L+2)/B$.
\end{itemize}

Given an input  vector \(\mathbf{x}=(x_1,\dots,x_{d_0})\in\mathbb{R}^{d_0}\) to the neural network, we define the \emph{induced input signal} as the computational input signal $f_{\mathbf{x}}$ satisfying $f_{\mathbf{x}}(v)=x_j$ whenever $v$ is in the $j$th interval of $\mathcal{C}^{\mathrm{in}}_{d_0}$. In the above construction, when writing ``the $j$th interval of the partition...'' we implicitly assume that the intervals are always sorted in increasing order.

\paragraph{Networks as MPNNs.} 
 Lemma \ref{lem:ker_equi} shows that the output of a $B$-strongly dense neural network can be represented as the output of an $L$-layer $B$-IR-MPNN applied to its induced computational kernel. 
 
\begin{restatable}[]{lemma}{MLPasKerMPNN}~\label{lem:ker_equi}
Let $(\weights,\bias)$ be the dense parameters of $\Theta_{(\weights,\bias)}\in\mathcal{NN}_{\mathrm{dense}}(B,L,d_0,d_L)$. Then, for any input \(\mathbf{x} \in \mathbb{R}^{d_0}\) and output channel $i\in[d_L]$, we have
\[
\Theta_{(\mathbf{W},\mathbf{b})}
(\mathbf{x})_i = \Phi_{B,L}\big(K_{(\mathbf{W},\mathbf{b})}
,f_{\mathbf{x}}\big)\big(v\big),
\]
whenever $v$ is in the $i$th interval of $\mathcal{C}^{\mathrm{out}}_{d_L}$.
\end{restatable}
\section{Model Compression}\label{sec:compression}
Modern deep learning models often contain millions or even billions of parameters, making them computationally expensive and memory-intensive. \emph{Model compression} -- a family of techniques including pruning~\citep{lecun1990optimal}, quantization~\citep{jacob2018quantization}, and distillation~\citep{hinton2015distilling} -- has shown experimentally that smaller networks can frequently match the performance of much larger ones. Our approach to compression differs fundamentally from these existing methods in objective. While these methods propose empirical heuristics to compress specific trained models, our goal is to provide a deterministic guarantee that \emph{any} network in a given class, regardless of whether it was trained or how it was initialized, can be approximated by a bounded-size network. This guarantee will be used in \cref{sec:implications} to establish a formal barrier to the universal approximation power of dense neural networks. The full proofs of results from this section are presented in \cref{app:compression}.

\paragraph{Strongly Dense Networks and The Weak Regularity Lemma.} The compression result is based on the weak regularity lemma stating that any graph can be approximated by a small graph in some sense. Accordingly, it uses the cut distance to measure the approximation of the large graph by its small approximating graph. Cut distance is a meaningful metric only for dense graphs that are not sparse, in the sense that any sparse graph (with much fewer edges than 
$n^2$) will always be seen as close to 0 in cut distance. The classical weak regularity lemma is defined for $[0,1]$-valued kernels, called graphons \citep{lovasz2012large}, and we extend it to computational kernels with values in $[-1,1]$. According to the classical theory, the adjacency matrix is normalized by $n$
 (the size of the graph) when treated as an operator, i.e., as a so-called graph shift operator. Hence,  when using the adjacency matrix to propagate the features in a message passing network, the adjacency matrix is normalized by $n$. This is exactly the normalization that we use in strongly dense neural networks. To summarize our normalization in the definition of strongly dense networks: We focus exactly on those MLPs that have dense computational kernels in the sense of the weak regularity lemma.

\paragraph{Computational Cut Distance.}
In order to define compression, we first need to define a measure of similarity between the compressed and original object. 
In our analysis, two computational kernels are considered equivalent if one can be obtained from the other using a ``permutation of the vertices'' that preserves the layer structure. Next, we formalize this idea more accurately.
Let $\mathcal{S}_{L}$ denote the set of measure-preserving bijections 
$\phi:[0,1]\to[0,1]$ such that each $U^{(\ell)}$, $\ell=1,\ldots,L-1$, is invariant under $\phi$, and $\phi(x)=x$ for every $x\in U^{(0)}\cup U^{(L)} \cup U^{(\mathrm{bias})}$.\footnote{More accurately, $\phi$ is a measure preserving bijection between two co-null sets of $[0,1]$, as our analysis is always up to sets of measure zero. Measure preserving means that $\mu(A)=\mu(\phi(A))$ for measurable subsets $A$ of the domain of $\phi$.} Define $J^\phi(x,y):=J(\phi(x),\phi(y))$. 
The \emph{computational cut distance} between the computational kernel-signals $(K,f),(J,g)$ is defined to be
\[
\delta_{\square}^{\mathrm{comp}}((K,f),(J,g))
:= \inf_{\phi \in \mathcal{S}_{L}}(\| K - J^\phi \|_\square +\| f - g \circ \phi \|_\square).
\]
We define similarly $\delta_{\square}^{\mathrm{comp}}(K,J)
:= \inf_{\phi \in \mathcal{S}_{L}}\| K - J^\phi \|_\square$.
The computational cut distance is a pseudometric. MPNNs were shown in \cite{signal23} to be Lipschitz continuous with respect to the cut norm, with a further extension in \cite{signal25}. In this work, we adapt the techniques of  \cite{signal23,signal25} to our setting to obtain an explicit bound on the Lipschitz constant of $B$-strongly dense ReLU networks w.r.t. the computational cut distance. 
\begin{restatable}
[Computational Cut Distance Lipschitz Continuity]{theorem}{LipschitzContinuityComp}\label{th:lipwithconstcomp}
Let $B>0$ and let $\Phi_{B,L}$ be the $L$-layer $B$-IR-MPNN. Let $(K,f),(J,g) \in \mathcal{CK}(B,L,d_0,d_L)\times \mathcal{CS}(L,d_0,d_L)$ be two computational kernel-signals with $f,g:[0,1]\to[0,1]$. Let 
    \[{\rm Out}_{K,f}:=\mathbbm{1}_{U^{(L)}}\Phi_{B,L}(K,f), \quad {\rm Out}_{J,g}:=\mathbbm{1}_{U^{(L)}}\Phi_{B,L}(J,g).\]
    Then
\[
\norm{{\rm Out}_{K,f}-{\rm Out}_{J,g}}_{\square}
\;\le\;
2^L B^L\,
\delta_{\square}^{\mathrm{comp}}\bigl((K,f),(J,g)\bigr).
\] 
\end{restatable}
Note that $\norm{{\rm Out}_{K,f}-{\rm Out}_{J,g}}_{\square}$ is a signal cut norm, as $({\rm Out}_{K,f}-{\rm Out}_{J,g}):[0,1]\rightarrow\mathbb{R}$.

\paragraph{Computational Kernel Weak Regularity Lemma.}
\cref{thm:wrlcompgraphon} establishes that any computational kernel, regardless of its original size or complexity, can be approximated by a smaller, structured kernel. Specifically, we utilize a variant of the weak regularity lemma \citep{frieze1999quick} for kernels \citep{lovas2007} to approximate any kernel by a step kernel whose complexity -- measured by the number of steps in its partition -- depends solely on the desired error $\epsilon$ in the computational cut distance. Building on techniques from \cite{signal23,signal25}, we refine this step kernel to preserve the specific neural bipartite layer structure, as well as the fixed input and output dimensions. 

\begin{restatable}{theorem}{WRMLPgraphon}~\label{thm:wrlcompgraphon}
Let $K\in \mathcal{CK}(B,L,d_0,d_L)$ be an $L$-layer computational kernel. For any $\epsilon>0$, there exists an $L$-layer computational kernel $K'\in \mathcal{CK}(B,L,d_0,d_L)$ with hidden dimension $$d=8ML\left\lceil {2^{2\lceil 16/\epsilon^2\rceil}}/{\epsilon} \right\rceil,$$ 
where $M$ is the least common multiple of $d_0$ and $d_L$, such that $\delta_{\square}^{\mathrm{comp}}(K,K')<\epsilon$.
\end{restatable}
This structural reduction allows us to represent the information of an arbitrarily wide network using a much smaller hidden dimension. When combined with the stability analysis of the forward pass presented in \cref{th:lipwithconstcomp} -- showing that $B$-IR-MPNNs are Lipschitz continuous with respect to the computational cut distance -- this leads directly to our main compression result. Corollary \ref{thm:arch_approx} provides a uniform bound on the size of the approximating network that is independent of the original network's width.

\begin{restatable}[Network Architecture Approximation]{corollaryy}{ArchApproxTwo}~\label{thm:arch_approx}
 Let $L,d_0,d_L\in\mathbb{N}$ with $L\ge 2$, let $B\ge L+2$, and $\epsilon > 0$. Then, for any network  $\Theta\in\mathcal{NN}_{\mathrm{dense}}(B,L,d_0,d_L)$ there is a network  $\Theta'\in\mathcal{NN}_{\mathrm{dense}}(B,L,d_0,d_L)$ with hidden dimension 
 \begin{align*}
&d= 8ML\left\lceil \tfrac{1}{\epsilon}\cdot  2^{2\lceil (16((L+2)d_L)^2(2B)^{2L})/\epsilon^2\rceil}(L+2)d_L(2B)^L \right\rceil,
 \end{align*}
such that $\|\Theta-\Theta'\|_{\infty}<\epsilon$.
\end{restatable}
\paragraph{Why Not Monte Carlo Compression?}
A natural alternative to our approach would be to approximate dense networks by randomly subsampling hidden layer neurons. This approach is often called \emph{randomized numerical linear algebra}. However, Monte Carlo sampling does not yield uniform guarantees: the approximation is only successful in high probability, and the event of success depends on the input to the network. Hence, different inputs would require different events of subsampled networks, and the intersection of all events in general can be the empty set,  preventing simultaneous control of the approximation error on all inputs. In contrast, the weak regularity lemma provides deterministic and input-independent compression guarantees.

\section{Implications for Approximation}\label{sec:implications}
The compression phenomenon established in \cref{sec:compression} reveals that any dense ReLU network -- no matter how large -- can be approximated by a bounded-size network in the computational cut distance. This finding hints at a fundamental limitation: if expressive power saturates at a finite level of architectural complexity, then sufficiently large networks cannot continue to gain representational capacity merely by adding width. In this section, we make this intuition precise by identifying explicit conditions under which dense deep ReLU networks \emph{fail to be universal approximators}. 
 
\paragraph{VC-Dimension Based Lower Bounds on Expressivity.}
The first step is to relate functional approximation capacity to combinatorial complexity. Recall that for a hypothesis class $\mathcal{H}$ of Boolean functions on $[0,1]^d$, the \emph{VC-dimension} is the largest number of points that $\mathcal{H}$ can shatter. When $\mathcal{H}$ consists of thresholded outputs of a ReLU network of fixed architecture but variable weights, Theorem 8.7 in \cite{Bartlett2009} provides the bound \(\mathrm{VCdim}(\mathcal{H}) \leq c \cdot W^2,\) where $W$ is the total number of parameters (weights and biases) and $c>0$ is a universal constant independent of the architecture. This implies that representational capacity grows at most quadratically in the number of parameters.
Theorem 3 of \cite{tightVC2017} gives a depth-dependent upper bound of order $O(WL\log W)$, which improves the quadratic bound asymptotically for fixed depth.

The next theorem provides a quantitative lower bound on the number of computational units necessary for universal approximation, building on \cite[Theorem~4(a)]{Yarotsky2017approximations}. Our contribution lies in adapting this bound to our setting, and expressing it in terms of a single implicit constant. Recall that ${\mathrm{Lip}}(d_0,d_L)$ is the space of $1$-Lipschitz continuous functions $[0,1]^{d_0} \to \mathbb{R}^{d_L}$ bounded by 1.

\begin{restatable}[]{theorem}{Yarotsky}
\label{thm:VCbounds}
For any \(\epsilon \in (0,1/3)\), a ReLU network architecture capable of approximating every function \( f \in {\mathrm{Lip}}(d_0,d_L) \) with error bounded by \(\epsilon\) must have at least \(c^{-1/2} (6\epsilon)^{-d_0/2}\) parameters, where $c$ is the absolute constant defined above.
\end{restatable}

%
%
\paragraph{Approximation Limits of Dense Neural Networks.}
We now connect the above expressivity lower bound with the compression result. Theorem \ref{thm:VCbounds} gives a lower bound on the number of parameters required to $\epsilon$-approximate certain Lipschitz functions, while Corollary \ref{thm:arch_approx} shows that neural networks that achieve $\epsilon$ approximation effectively have a uniformly bounded number of parameters.  However, as it turns out, in certain settings, the lower bound in Theorem \ref{thm:VCbounds} is higher than the upper bound entailed by Corollary \ref{thm:arch_approx}. This shows that dense neural networks are not universal approximators of ${\mathrm{Lip}}(d_0,d_L)$.

\begin{restatable}[Expressivity Bound of Strongly Dense Neural Networks]{theorem}{ExpressivityBound}
\label{thm:ExpressivityBoundsOne}
Let $L,d_L\in\mathbb{N}$ with $L\ge 2$, let $B\ge L+2$, and let $c$ be the constant from \cref{thm:VCbounds}. Let $d_0$ satisfy
\begin{equation*}
d_0 \geq 17\log_2\!\bigl(
c^{1/2}L^3(L+2)^2d_L^4(2B)^{2L}
\bigr)
+ 17\cdot 2^{14}(L+2)^2d_L^2(2B)^{2L} + 306.
\end{equation*}
Then $\mathcal{NN}_{\mathrm{dense}}(B,L,d_0,d_L)$ is not a universal approximator of ${\mathrm{Lip}}(d_0,d_L)$.
\end{restatable}
We interpret \cref{thm:ExpressivityBoundsOne} as follows. Since general neural networks of a fixed depth are universal approximators of $\mathrm{C}(d_0,d_L)$, they are also universal approximators of $\mathrm{Lip}(d_0,d_L)$. One might expect that neural networks with linear layers bounded by some universal constant $C>1$ in their induced infinity norm, i.e. $\|W^{(\ell)}\|_{\infty}<C$, are also universal approximators of $\mathrm{Lip}(d_0,d_L)$.\footnote{In fact, as far as we know, the following is an open question: There exists a constant $C>1$, such that the space of neural networks with linear layers $W^{(\ell)}$ having induced infinity norm $\|W^{(\ell)}\|_{\infty}<C$ is a universal approximator of $\mathrm{Lip}(d_0,d_L)$. We note that networks with sort activations are known to satisfy this property \citep{Anil2018SortingOL,Tanielian2020ApproximatingLC}.
} In our definition of strongly dense neural networks, if one chooses $B=(L+2)C$ then the infinity norms of the linear layers are bounded by $C$. Hence, one might naively anticipate that $\mathcal{NN}_{\mathrm{dense}}(B,L,d_0,d_L)$ is a universal approximator of ${\mathrm{Lip}}(d_0,d_L)$. However, \cref{thm:ExpressivityBoundsOne} shows that this is false, given that the input dimension is large enough (but still independent of $n$).  
%

\paragraph{Overview of the Theoretical Approach and Implication.} We summarize our theoretical approach as follows. 
Feed forward on an MLP can be implemented as a message passing network on the
corresponding computational kernel (\cref{lem:ker_equi}). MPNNs are continuous
with respect to the computational cut distance (\cref{th:lipwithconstcomp}), and
the weak regularity lemma asserts that any large computational kernel can be
approximated in this distance by a small one (\cref{thm:wrlcompgraphon}).
Combining the three leads to our compression bound (\cref{thm:arch_approx}): any
large strongly dense MLP is approximated, as a function, by a small one. Together
with the parameter lower bound of \citet{Yarotsky2017approximations}
(\cref{thm:VCbounds}), this yields \cref{thm:ExpressivityBoundsOne}: at high
enough input dimension, strongly dense networks cannot be universal approximators. The philosophical implication is that expressive universality (the ability to represent “anything”) can only emerge on low dimensional data for strongly dense networks.

\paragraph{Connection to Double Descent and Modern Deep Learning.}
Our results in \cref{thm:ExpressivityBoundsOne} offer, partially, a theoretical lens through which to view the double descent phenomenon. Classical bias-variance logic suggests that as the parameter count surpasses the sample size, the model's variance should diverge, causing the model to overfit the data and impair the generalization capacity. However, we show that for strongly dense networks, the realizable capacity actually plateaus, which supports the empirical observation that there is no unbounded growth of variance. We note that in practical deep learning, the density of the neural network is often implicitly  enforced through techniques such as weight decay or layer normalization.

%
%
\section{Conclusion}
We presented a framework linking graph-regularity-based compression of dense ReLU networks with VC-dimension lower bounds on expressivity. The combination yields explicit conditions under which depth- and weight-bounded dense networks cease to be universal, revealing a saturation effect: beyond a certain scale, width no longer increases representational capacity. 
These findings suggest that expressivity is often governed more by architectural structure than by size alone. While our present results apply to fully connected feedforward architectures, extending our compression-based analysis to structured families -- such as convolutional or attention-based networks -- offers a promising direction for developing a unified theory for expressivity saturation in deep learning.
\section*{Acknowledgement}

S.J. is supported by an Alexander von Humboldt professorship. R.L. is supported by a grant from the United States-Israel Binational Science Foundation (BSF), Jerusalem, Israel, and the United States National Science Foundation (NSF), (NSF-BSF, grant No.~2024660), and by the Israel Science Foundation (ISF grant No.~1937/23).

\bibliographystyle{plainnat}
\bibliography{yourbibfile}
\appendix
\crefalias{section}{appendix}
\section{Proof of Deep Neural Networks as MPNNs}\label{app:comp_graph}
Lemma \ref{lem:ker_equi} shows that the output of a $B$-strongly dense neural network can be represented as the output of an $L$-layer $B$-IR-MPNN applied to its induced computational kernel. To prove Lemma \ref{lem:ker_equi}, we first show Lemma \ref{lem:const_bias_sig}, which states that the $B$-IR-MPNN applied to a computational kernel-signal is stationary on $U^{(\mathrm{bias})}$.
\begin{lemma}\label{lem:const_bias_sig}
Let $(K_{(\weights,\bias)},f)\in\mathcal{CK}(B,L,d_0,d_L)\times \mathcal{CS}(L,d_0,d_L)$ be a computational kernel-signal, where $f$ is a computational input signal. 
 Then, the hidden representations of the $B$-IR-MPNN satisfy:
\[
f^{(\ell)}|_{U^{(\mathrm{bias})} }\equiv 1,\quad \mathrm{for} \quad \ell=0,\ldots,L.
\]
That is, the hidden representations are constant and share the same value on $U^{({\mathrm{bias}})}$.    
\end{lemma}
\begin{proof}
We prove the claim using induction. Recall that by the definition of computational kernel for any $v\in U^{(\mathrm{bias})}$:
$$K(v,u)=\begin{cases}
     (L+2)/B \quad {\mathrm{if}} \quad u\in U^{(\mathrm{bias})},\\
    0\qquad\qquad\quad {\mathrm{otherwise}}. 
\end{cases}$$
\paragraph{Induction Base.} By the definition of a computational input signal and $f^{(0)}=f$, we have
$$f^{(0)}|_{U^{(\mathrm{bias})} }\equiv 1.$$
\paragraph{Induction Assumption.} We assume that $f^{(\ell-1)}|_{U^{(\mathrm{bias})} }\equiv 1$ for $0<\ell\le L$.
\paragraph{Induction Step.} 
For $1\le\ell<L$, by the definition of the $B$-IR-MPNN and the induction assumption, for any $v\in U^{(\mathrm{bias})}$:
\begin{align*}
f^{(\ell)}(v)& := 
B\cdot{\mathrm{ReLU}}\!\left(\int_{[0,1]} K(v,u)f^{(\ell-1)}(u)du\right)\\&=B\cdot{\mathrm{ReLU}}\!\left(\int_{U^{(\mathrm{bias})}} (L+2)/Bf^{(\ell-1)}(u)du\right)\\&=(L+2)\cdot{\mathrm{ReLU}}\!\left(\int_{U^{(\mathrm{bias})}}f^{(\ell-1)}(u)du\right)
\\&=(L+2)\cdot{\mathrm{ReLU}}\!\left(\int_{U^{(\mathrm{bias})}}du\right)
\\&=(L+2)\cdot\mu\left(U^{(\mathrm{bias})}\right)=1.
\end{align*} 
For $\ell=L$, the same calculation without ReLU gives $f^{(L)}|_{U^{(\mathrm{bias})}}\equiv1$.
This concludes the proof.
\end{proof}
Now, we use Lemma \ref{lem:const_bias_sig} to show Lemma \ref{lem:ker_equi}.
\MLPasKerMPNN*
\begin{proof}  We now show that for $\ell=0,\ldots,L$ the activation of the $\ell$th hidden layer $\mathbf{h}^{(\ell)}_i$ in channel $i\in[d_\ell]$ of the network $\Theta_{(\weights,\bias)}$ is equal to the hidden $B$-IR-MPNN representation $f^{(\ell)}(v)$ whenever $v$ is in the $i$th interval of
$\cP^{(\ell)}$, when 
$$\cP^{(\ell)}:=\begin{cases}
 \mathcal{C}^{\mathrm{in}}_{d_0},\quad \mathrm{if} \quad \ell=0,\\
 \mathcal{R}_d^{(\ell)} \quad \mathrm{if} \quad 0<\ell<L ,\\
 \mathcal{C}^{\mathrm{out}}_{d_L} \quad \mathrm{if} \quad \ell=L.  \\
\end{cases}$$
Recall that by the definition of the computational kernel, for any $v$ in the $i$th interval of $\cP^{(\ell)}$, $\ell\in [L]$:
$$K(v,u)=\begin{cases}
     b_i^{(\ell)}/B \quad {\text{ if }} \quad u\in U^{(\mathrm{bias})}\text{ and }v \text{ is in the }i\text{th interval of } \mathcal{P}^{(\ell)},\\
     W_{i,j}^{(\ell)}/B {\text{ }\text{  if }} \quad u \text{ is in the $j$th interval of }\cP^{(\ell-1)}\text{ and }v \text{ is in the }i\text{th interval of } \mathcal{P}^{(\ell)},\\
    0\qquad\quad{\text{ otherwise}}. 
\end{cases}$$
We prove the statement inductively.
\paragraph{Induction Base.}
Let  $\ell=0 $ and $ i\in[d_0]$, then, for any $v$ in the $i$th interval of $\cP^{(0)}$ and $\mathbf{x}=(x_1,\ldots,x_{d_0})\in\mathbb{R}^{d_0}$
\[
f^{(0)}(v)= f_\mathbf{x}(v) = x_i=\mathbf{h}_i^{(0)} .
\]

\paragraph{Induction Assumption.} 
For layers $\ell = 1,\dots,L$, channels $i\in[d_{\ell-1}]$ and for any $v$ in the $i$th interval of $\cP^{(\ell-1)}$, assume that
\[
f^{(\ell-1)}(v)=\mathbf{h}_i^{(\ell-1)} .
\]
\paragraph{Induction Step.}
Denote the $j$th interval of the partition $\cP^{(\ell)}$ by $I^{(\ell)}_j$. Then, for $1\le\ell<L$ and any $v$ in the $i$th interval of the partition $\cP^{(\ell)}$:
\begin{align*}
f^{(\ell)}(v) &:= 
B\cdot{\mathrm{ReLU}}\!\left(\int_{[0,1]} K(v,u)f^{(\ell-1)}(u)du\right)
\\&
=B\cdot{\mathrm{ReLU}}\!\left(\sum_{I_{j}^{(\ell-1)}\in\cP^{(\ell-1)}}\int_{I_{j}^{(\ell-1)}} (W^{(\ell)}_{i,j}/B)f^{(\ell-1)}(u)du+\int_{U^{(\mathrm{bias})}}(b_i^{(\ell)}/B)f^{(\ell-1)}(u)du\right)
\\&
={\mathrm{ReLU}}\!\left(\sum_{I_{j}^{(\ell-1)}\in\cP^{(\ell-1)}}\int_{I_{j}^{(\ell-1)}} W^{(\ell)}_{i,j}f^{(\ell-1)}(u)du+\int_{U^{(\mathrm{bias})}}b_i^{(\ell)}f^{(\ell-1)}(u)du\right)=:(*)
\end{align*}
By the induction assumption $f^{(\ell-1)}(v)=\mathbf{h}^{(\ell-1)}_i$ on any $v$ in the $i$th interval of $\cP^{(\ell-1)}$ and since by definition $f_\mathbf{x}|_{U^{(\mathrm{bias})}}\equiv 1$, then, $f^{(\ell-1)}|_{U^{(\mathrm{bias})}}\equiv 1$ by Lemma \ref{lem:const_bias_sig}. Thus
\begin{align*}
(*) &
={\mathrm{ReLU}}\!\left(\sum_{j=1}^{d_{\ell-1}}\int_{I_{j}^{({\ell-1})}} W^{({\ell})}_{i,j}\mathbf{h}^{(\ell-1)}_jdu+\int_{U^{(\mathrm{bias})}}b_i^{({\ell})}du\right)
\\&={\mathrm{ReLU}}\!\left(\sum_{j=1}^{d_{\ell-1}}W^{(\ell)}_{i,j}\mathbf{h}^{(\ell-1)}_j\cdot\mu\left(I_{j}^{({\ell-1})}\right)+b_i^{(\ell)}\mu\left(U^{(\mathrm{bias})}\right)\right)=:(**)
\end{align*}
Recall that $\mu\left(U^{(\mathrm{bias})}\right)=1/(L+2)$ and $\mu\left(I_{j}^{({\ell-1})}\right)=1/(d_{\ell-1}(L+2))$ by the definition of the computational kernel. Therefore
\begin{align*}
(**)&
={\mathrm{ReLU}}\!\left(\sum_{j=1}^{d_{\ell-1}}W^{({\ell})}_{i,j}\mathbf{h}^{(\ell-1)}_j\int_{I_{j}^{({\ell-1})}} du+b_i^{({\ell})}\int_{U^{(\mathrm{bias})}}du\right)
\\&
={\mathrm{ReLU}}\!\left(\sum_{j=1}^{d_{\ell-1}}W^{({\ell})}_{i,j}\mathbf{h}^{(\ell-1)}_j\int_{I_{j}^{({\ell-1})}} du+\frac{1}{L+2}b_i^{({\ell})}\right)
\\&
=\mathrm{ReLU}\left(\frac{1}{d_{\ell-1}(L+2)}
\sum_{j=1}^{d_{\ell-1}}\left(W^{(\ell)}_{ij}\mathbf{h}^{({\ell-1})}_j + b^{(\ell)}_i\right)\right)=\mathbf{h}^{(\ell)}_i.
\end{align*}
For $\ell=L$, the same calculation without ReLU gives $f^{(L)}(v)=\mathbf{h}^{(L)}_i$ for $v$ in the $i$th output interval.
This concludes the proof since
$$\Theta_{(\mathbf{W},\mathbf{b})}
(\mathbf{x})_i =\mathbf{h}^{(L)}_i=f^{(L)}(v)= \Phi_{B,L}\big(K_{(\mathbf{W},\mathbf{b})}
,f_{\mathbf{x}}\big)\big(v\big),$$
whenever $v$ is in the $i$th interval of $\mathcal{C}_{d_L}^{\mathrm{out}}$.
\end{proof}
\section{Proofs of Model Compression}\label{app:compression}
Here we prove our Lipschitz continuity of $B$-IR-MPNN (Theorem \ref{th:lipwithconstcomp}), our weak regularity lemma (Theorem \ref{thm:wrlcompgraphon}), and our compression result (Corollary \ref{thm:arch_approx}).
\paragraph{Lipschitz Continuity of MPNNs.}
MPNNs were shown to be Lipschitz continuous with respect to the cut norm in \cite{signal23}, with a further extension in \cite{signal25}. In this work, we adapt the techniques developed in \cite{signal23} to our setting in order to obtain an explicit bound on the Lipschitz constant.

We first recall some useful lemmas.

\begin{lemma} [\cite{signal23}, Equation (9)]
    \label{lem:cut1}
    Let $f:[0,1]\rightarrow \mathbb{R}$ be measurable. Then
    \[\frac{1}{2}\norm{f}_1\leq \norm{f}_{\square} \leq \norm{f}_1.\]
\end{lemma}

Denote by $L^+[0,1]$ the space of measurable functions $f:[0,1]\rightarrow [0,1]$. 

\begin{lemma} [\cite{signal23}, Lemma F.5/\cite{lovasz2012large}, Lemma 8.10]
    \label{lem:cutfg}
    Let $Q:[0,1]^2\rightarrow\mathbb{R}$ be measurable. Then
    \[\norm{Q}_{\square}=\sup_{f,g\in L^+[0,1]}\abs{\int_{[0,1]^2}f(x)Q(x,y)g(y)\mathrm{d}x\mathrm{d}y},\]
    where the supremum is attained for some $f,g\in L^+[0,1]$ with values in $\{0,1\}$.
\end{lemma}

\begin{restatable}
[Cut Norm Lipschitz Continuity]{lemma}{LipschitzContinuity}\label{lem:lipwithconst}
Let $B>0$, let $\Phi_{B,L}$ be the $L$-layer $B$-IR-MPNN, and let $(K,f),(J,g)$ 
be two attributed kernels with $f,g:[0,1]\to[0,1]$. Then,
\begin{align*}
   &\norm{\Phi_{B,L}(K, f) - \Phi_{B,L}(J, g)}_{\square}\leq 2^L B^L(\norm{f-g}_\square
+
\norm{K-J}_\square).
\end{align*}
\end{restatable}

\begin{proof} 
Recall that for a kernel-signal $(T,q)$
\begin{align*}
\Phi_{B,L}(T,q)(x)
&:=q^{(L)}(x),
\end{align*}
and for $\ell=1,\ldots,L-1$,
\begin{align*}
q^{(\ell)}(x):= B\cdot\,{\mathrm{ReLU}}\!\left(\int_{[0,1]} T(x,y) q^{(\ell-1)}(y)\, \mathrm{d}y\right).
\end{align*}
The output is computed by the same formula with $\ell=L$, without ReLU.

We first show that for every $\ell$, $\norm{f^{(\ell)}}_{\infty},\norm{g^{(\ell)}}_{\infty}\leq B^{\ell}$.
Indeed, for any $x\in[0,1]$
\[|f^{(\ell)}(x)| \leq  B \abs{\int_{[0,1]} K(x,y) f^{(\ell-1)}(y) \mathrm{d}y}.\]
Hence, by H\"older's inequality
\[|f^{(\ell)}(x)|\leq B \| f^{(\ell -1)}\|_1 \leq B\| f^{(\ell-1)}\|_{\infty},\]
so we also have
\[\|f^{(\ell)}\|_{\infty} \leq B \| f^{(\ell-1)} \|_{\infty}.\]
By solving a recurrence sequence, and using $\norm{f^{(0)}}_{\infty}\leq 1$, we get $\norm{f^{(\ell)}}_{\infty}\leq B^{\ell}$.

Since ReLU is monotone and $1$-Lipschitz, the positive and negative parts of a difference cannot increase after applying ReLU. Thus, the following inequality holds for every $\ell\in[L]$, with equality for the linear output layer.
\begin{align*}
\norm{f^{(\ell)}-g^{(\ell)}}_\square
&\le
B\norm{\int K(\cdot,y) f^{(\ell-1)}(y)\,\mathrm{d}y
-
\int J(\cdot,y) g^{(\ell-1)}(y)\,\mathrm{d}y}_\square.
\end{align*}
We decompose the expression inside the norm as
\begin{align*}
\int K(x,y) & f^{(\ell-1)}(y)\,\mathrm{d}y
-
\int J(x,y) g^{(\ell-1)}(y)\,\mathrm{d}y
\\&=
\int J(x,y)\bigl(f^{(\ell-1)}(y)-g^{(\ell-1)}(y)\bigr)\,\mathrm{d}y
+
\int \bigl(K(x,y)-J(x,y)\bigr) f^{(\ell-1)}(y)\,\mathrm{d}y.
\end{align*}
By the triangle inequality,
\begin{align}
\norm{f^{(\ell)}-g^{(\ell)}}_\square\le{}&
B\norm{\int J(x,y)(f^{(\ell-1)}(y)-g^{(\ell-1)}(y))\,\mathrm{d}y}_\square\label{eq:InProofLip11}
\\&+
B\norm{\int (K-J)(x,y) f^{(\ell-1)}(y)\,\mathrm{d}y}_\square.\label{eq:InProofLip12}
\end{align}
We bound \cref{eq:InProofLip11}, using Lemma \ref{lem:cut1}, by
\begin{align*}
\norm{\int J(x,y)(f^{(\ell-1)}(y)-g^{(\ell-1)}(y))\,\mathrm{d}y}_\square&\leq \norm{\int J(x,y)(f^{(\ell-1)}(y)-g^{(\ell-1)}(y))\,\mathrm{d}y}_1
\\&\leq \int\Big|\int J(x,y)(f^{(\ell-1)}(y)-g^{(\ell-1)}(y))\,\mathrm{d}y\Big| dx
\\&\leq \int\int \Big|J(x,y)(f^{(\ell-1)}(y)-g^{(\ell-1)}(y))\Big|\mathrm{d}y \mathrm{d}x.
\end{align*}
Hence, by H\"older's inequality,
\begin{align*}
\norm{\int J(x,y)(f^{(\ell-1)}(y)-g^{(\ell-1)}(y))\,\mathrm{d}y}_\square&\leq \int 1\int \Big|(f^{(\ell-1)}(y)-g^{(\ell-1)}(y))\Big|\mathrm{d}y dx\\
&\le
\norm{f^{(\ell-1)}-g^{(\ell-1)}}_1\\
&\le
2\norm{f^{(\ell-1)}-g^{(\ell-1)}}_\square
\end{align*}
where the last inequality is due to Lemma \ref{lem:cut1}.

We bound \cref{eq:InProofLip12} as follows. Let $q=f^{(\ell-1)}/B^{\ell -1}$. By the application of ReLU, and by the bound $\norm{f^{(\ell-1)}}_{\infty}\leq B^{\ell-1}$,  the function $f^{\ell-1}$ maps $[0,1]$ to $[0,B^{\ell-1}]$, so $q:[0,1]\rightarrow[0,1]$. Let $Q:=K-J:[0,1]^2\rightarrow [-2,2]$. Our goal is to bound 
$\norm{\int  Q(x,y)q(y)dy}_{\square}$.
Let $S\subset [0,1]$ be the set for which the cut norm is realized. The existence of such an $S$ is guaranteed by Lemma \ref{lem:cutfg}. Hence, 
\[\norm{\int  Q(x,y)q(y)dy}_{\square}=\left|\int\int \mathbbm{1}_S(x)Q(x,y)q(y)dy dx\right|.\]
Using Lemma \ref{lem:cutfg},  we get
\[|\int\int \mathbbm{1}_S(x)Q(x,y)q(y)dy dx| \leq \norm{Q}_{\square}.\]

Combining the bounds, we obtain
\[\norm{f^{(\ell)}-g^{(\ell)}}_\square
\le
2B\norm{f^{(\ell-1)}-g^{(\ell-1)}}_\square
+
B^{\ell}\norm{K-J}_\square.\]
 Hence,
\begin{align*}
\norm{f^{(\ell)}-g^{(\ell)}}_\square
\le
2B\norm{f^{(\ell-1)}-g^{(\ell-1)}}_\square
+
B^{\ell}\norm{K-J}_\square.
\end{align*}
Now, denote by $c^{(\ell)}$  bounds on $\norm{f^{(\ell)}-g^{(\ell)}}_\square$, for $\ell=1,\ldots,L$, which satisfy
\[c^{(\ell)}
=
2Bc^{(\ell-1)}
+
B^{\ell}r,\]
where $r=\norm{K-J}_\square$.
Solving this recurrence sequence, we get
\[c^{(\ell)}=2^{\ell}B^{\ell}c^{(0)} + B^{\ell}(2^{\ell}-1)r \leq 2^{\ell}B^{\ell}(c^{(0)}+r).\]
Hence,
\[\norm{f^{(\ell)}-g^{(\ell)}}_\square
\le
2^{\ell} B^{\ell}(\norm{f^{(0)}-g^{(0)}}_\square
+ \norm{K-J}_\square) .\]
\end{proof}
The following result is a consequence of Lemma \ref{lem:lipwithconst}.

\LipschitzContinuityComp*

\begin{proof}
Consider an arbitrary measure preserving bijection $\phi\in\mathcal{S}_{L}$ that preserves the layer structure. Applying Lemma \ref{lem:lipwithconst} to the pair
$(K,f)$ and $(J^{\phi},g\circ{\phi})$ yields
\begin{align*}
   &\norm{\Phi_{B,L}(K, f) - \Phi_{B,L}(J^{\phi},g\circ{\phi})}_{\square}\leq  2^LB^L \big( \norm{K-J^{\phi}}_{\square} + \norm{f-g\circ\phi}_{\square} \big). 
\end{align*}
The $B$-IR-MPNN is equivariant under measure-preserving relabelings:
\[
\Phi_{B,L}(J^\phi,g\circ\phi)
=
\Phi_{B,L}(J,g)\circ\phi.
\]
Since every $\phi\in\mathcal S_L$ fixes $U^{(L)}$ pointwise, and restricting a signal to $U^{(L)}$ cannot increase its cut norm, the preceding inequality implies
\begin{equation}\label{eq:et5}
\norm{{\rm Out}_{K,f}-{\rm Out}_{J,g}}_\square
\le
2^LB^L\left(
\norm{K-J^\phi}_\square+
\norm{f-g\circ\phi}_\square
\right)
\end{equation}
for every $\phi\in\mathcal S_L$. Taking the infimum over $\phi\in\mathcal S_L$ gives the claimed bound.
\end{proof}

\paragraph{The Regularity Lemma for computational Kernels.}Let us recall some definitions. Let $\mu$ denote the Lebesgue measure on $[0,1]$. A \emph{partition} of $[0,1]$ is a sequence $\mathcal{P}_k = \{P_1, \dots, P_k\}$ of disjoint measurable sets such that $\bigcup_{j=1}^k P_j = [0,1]$. The partition is called an \emph{equipartition} if $\mu(P_i) = \mu(P_j)$ for all $i,j \in [k]$. We write $\mathds{1}_S$ for the indicator function of a set $S$. We define step functions as follows.   
\begin{definition}
\label{def:step}
Given a partition $\mathcal{P}_n$, define the space $\mathcal{S}^2_{\mathcal{P}_n}$ of \emph{step functions} $[0,1]^2\mapsto\mathbb{R}$ over the partition $\mathcal{P}_n$ to be the space of functions of the form
\begin{equation*}
F(x,y)=\sum_{i,j\in [n]\times[n]} c_{i,j}\mathds{1}_{P_{i}\times P_{j}}(x,y), 
\end{equation*}
for any choice of $\{c_{i,j}\in \mathbb{R}\}_{i,j\in [n]\times[n]}$.
\end{definition}
Denote the space of kernels $K:[0,1]^2\rightarrow[-1,1]$ by $\cW_1$. Notice that any step kernel with respect to $\mathcal{P}_n$ is in the intersection $\mathcal{S}^2_{\mathcal{P}_n}\cap\cW_1$. We define the projection of a kernel onto a partition as its blockwise average over the parts of the partition.
\begin{definition}
\label{def:proj}
Let $\mathcal{P}_n=\{P_1,\ldots,P_n\}$ be a partition of $[0,1]$, and $K\in\cW_1$ be a kernel. The \emph{projection} of $K$ upon $\mathcal{S}^{2}_{\mathcal{P}_n}$ is the step kernel $K_{\mathcal{P}_n}$ that attains the value
\begin{align*}
  & K_{\mathcal{P}_n}(x,y) = \frac{1}{\mu(P_i)\mu(P_j)}\int_{[0,1]^2}K(x,y)\mathds{1}_{P_i\times P_j}(x,y)\mathrm{d}x\mathrm{d}y , 
\end{align*}
    for every $(x,y)\in P_i\times P_j$ and $1\leq i,j\leq n$. 
\end{definition}
The projection is also called the \emph{stepping operator}. The following theorem, provided in \cite{signal25}, Appendix B, is an adaptation of the ``analyst's version'' of the  \emph{weak regularity lemma} for kernels, originally introduced in \cite{lovas2007}. 
\begin{theorem}[\cite{signal25}, Theorem B.5]
\label{lem:gs-reg-lem-refine}
Let $\epsilon>0$. For every kernel $K \in \cW_1$ there exists a partition $\cP_k$ of $[0,1]$ into $  k= 2^{2\lceil 1/\epsilon^2\rceil }$ sets, and a step kernel $K_{k}\in \mathcal{S}_{\cP_k}^2\cap \cW_1$, such that
\begin{equation*}
  \norm{K-K_{k}}_{\square}\leq \epsilon.
\end{equation*} 
\end{theorem}

The following lemma is given in \cite{signal25} as Lemma B.10. 
\begin{lemma}\label{lem:attain}
Let $\mathcal{P}_n=\{P_1,\ldots,P_n\}$ be a partition of $[0,1]$, and
    Let $V,R\in \mathcal{S}_{\mathcal{P}_n}^{2}\cap\cW_1$. Then, the supremum of
    \begin{equation*}     \sup_{S,T\subset [0,1]} \abs{\int_S\int_T\big(V(x,y)-R(x,y)\big)dxdy}
\end{equation*}
    is attained for $S,T$ of the form
    \[S = \bigcup_{i\in s}P_i\ , \quad T = \bigcup_{j\in t}P_j,\]
    where $t,s\subset [n]$.
\end{lemma}
As part of the proof of Theorem \ref{thm:wrlcompgraphon}, we first show that any computational kernel can be approximated by a coarse step kernel, where the partition underlying the coarse kernel is not an equipartition in   general. Hence, the last step of the proof involves approximating the coarse kernel by one based on an equipartition. For this, we recall an equatizing lemma from \cite{signal23}. 
Inspecting the proof of Lemma B.1 from \cite{signal23}, we can reformulate the lemma as follows.

\begin{lemma}[\cite{signal23}, Equitizing partition]
\label{l:equipartition}
Let $\mathcal{P}_k$ be a partition of $[0,1]$ into $k$ sets (generally not of the same measure). Then, for any $n>k$ there exists an equipartition $\mathcal{E}_n$ of $[0,1]$ into $n$ sets which satisfies the following. 
 For each $j\in[k]$, let $\mathcal{E}_n^j$ be the set of all parts of $\mathcal{E}_n$ that are subsets of $P_j\in\mathcal{P}_k$. Denote \[\mathcal{T}_n:=\bigcup_{j=1}^k\mathcal{E}_n^j,\]
 and $\mathcal{R}_n=\mathcal{E}_n\setminus\mathcal{T}_n$. Then, there exists an integer $h\leq k$ such that $\mu(\cup\mathcal{R}_n)=h/n$ and $\mu(\cup\mathcal{T}_n)=1-h/n$. We call $\mathcal{T}_n$ the \emph{refinement parts} of $\mathcal{E}_n$ and $\mathcal{R}_n$ the \emph{remainder parts}.
\end{lemma}

For completeness, we repeat the proof.

\begin{proof}
Let $\mathcal{P}_k=\{P_1,\ldots,P_k\}$ be a measurable partition of $[0,1]$ and let $n>k$.
For each $i\in[k]$, subdivide $P_i$ into measurable sets \(P_{i,1},\ldots,P_{i,m_i}
\)
as follows. If $\mu(P_i)<1/n$, set $m_i=1$ and $P_{i,1}=P_i$. Otherwise, choose $m_i\ge2$ such that \(\mu(P_{i,j})=\frac{1}{n}\) for $j\in[m_i-1]$ and $\mu(P_{i,m_i})<\frac{1}{n}$, allowing the last set to be empty.
We refer to $P_{i,m_i}$ as the remainder of $P_i$. Define the sequence of parts of measure $1/n$ to be
\begin{align*}
\mathcal{Q}
:=
\{P_{1,1},\ldots,P_{1,m_1-1},
   P_{2,1},\ldots,P_{2,m_2-1},
   \ldots,
   P_{k,1},\ldots,P_{k,m_k-1}\},
\end{align*}
where for indices $i$ with $m_i=1$ there is no   contribution in $\mathcal{Q}$.
Let $l:=|\mathcal{Q}|$. By construction, \(\mu\Big(\bigcup \mathcal{Q}\Big)=\frac{l}{n}.\) Let \(\Pi:=\bigcup_{i=1}^k P_{i,m_i}.\)
Then \(
\mu(\Pi)=1-\frac{l}{n}=\frac{h}{n}
\), for \(h:=n-l\). Since each $P_{i,m_i}$ has measure at most $1/n$ and there are at most $k$ such sets, we have $h\le k$.

Next, partition $\Pi$ into $h$ measurable sets \(\Pi_1,\ldots,\Pi_h\) each of measure exactly $1/n$. Define
\[
\mathcal{E}_n
:=
\mathcal{Q}\,\cup\,\{\Pi_1,\ldots,\Pi_h\}.
\]
Then $\mathcal{E}_n$ is a partition of $[0,1]$ into $n$ sets, each of measure $1/n$.
Let
\[
\mathcal{T}_n:=\mathcal{Q},
\qquad
\mathcal{R}_n:=\{\Pi_1,\ldots,\Pi_h\}.
\]
By construction,
\[
\mu\Big(\bigcup\mathcal{T}_n\Big)=1-\frac{h}{n},
\qquad
\mu\Big(\bigcup\mathcal{R}_n\Big)=\frac{h}{n},
\qquad
h\le k,
\]
and each element of $\mathcal{T}_n$ is contained in some $P_i\in\mathcal{P}_k$.
\end{proof}
The result extends to any interval $[a,b]$, with $a<b$, by affine rescaling, with all Lebesgue measures multiplied by $b-a$.
The following lemma is a basic result from measure theory.
\begin{lemma}
\label{lem:bijection2interval}
    For any partition $\mathcal{P}$ of $[a,b]$, where $a<b$, there is a measure-preserving bijection $\phi$ between co-null subsets of $[a,b]$ whose inverse maps $\mathcal{P}$ into an interval partition $\mathcal{J}$ of $[a,b]$. Namely,  every $J\in\mathcal{J}$ is of the form $J=\phi^{-1}(P)$, up to a null-set, for some $P\in\mathcal{P}$.
\end{lemma}
\begin{proof}
First, let $[a,b] = [0,1]$. Let $\mathcal{P} = \{P_1, P_2, \dots\}$ be the partition of $[0,1]$ into measurable, disjoint sets, with \(\sum_i \mu(P_i) = 1.\) For each $i$, let $q_i = \mu(P_i)$. Define intervals
\[
J_1 = [0, q_1], \quad 
J_2 = [q_1, q_1 + q_2], \quad \dots, \quad
J_n = \Big[\sum_{k=1}^{n-1} q_k, \sum_{k=1}^{n} q_k\Big].
\]
This gives an interval partition $\mathcal{J} = \{J_1, J_2, \dots\}$ of $[0,1]$. For any two measurable sets $A,B \subset [0,1]$ of equal measure, there exists a measure-preserving bijection $f: A \to B$, up to null sets. Applying this to each $P_i$ and $J_i$, we obtain bijections \(\phi_i: J_i \to P_i\)
which are measure-preserving (up to null sets). Define \(\phi(x) = \phi_i(x)\) for \(x \in J_i.\) Then $\phi: [0,1] \to [0,1]$ is a measure-preserving bijection up to null sets. Moreover, by construction, each $J_i = \phi^{-1}(P_i)$ up to a null set. If the original interval is $[a,b]\neq[0,1]$, let $T:[0,1]\to[a,b]$ be the affine map $T(x) = a + (b-a)x$. Then the map \(\tilde \phi := T \circ \phi \circ T^{-1} : [a,b] \to [a,b]\) is a measure-preserving bijection (up to null sets) whose inverse maps $\mathcal{P}$ to an interval partition of $[a,b]$, completing the proof. For any measurable set $B \subset [a,b]$, let $A := T^{-1}(B) \subset [0,1]$. Then \(\tilde \phi^{-1}(B) = T(\phi^{-1}(A)),\) and the measure on $[a,b]$ satisfies \(\mu_{[a,b]}(\tilde \phi^{-1}(B)) = (b-a)\, \mu(\phi^{-1}(A)) = (b-a)\, \mu(A) = \mu_{[a,b]}(B),\) so $\tilde \phi$ is measure-preserving. Furthermore, by construction, $\tilde \phi^{-1}$ maps $\mathcal{P}$ to an interval partition of $[a,b]$, completing the proof.
\end{proof}
The proof strategy of \cref{thm:wrlcompgraphon} uses techniques from \cite{signal23}, Corollary B.11 and \cite{signal25}, Theorem 4.9, both extending techniques from \cite{lovas2007}. We begin by applying the weak regularity lemma for kernels (\cref{lem:gs-reg-lem-refine}) to obtain an initial approximation of the graph-induced kernel. We then project this kernel onto the partition provided by \cref{l:equipartition}, yielding a structured kernel approximation that corresponds to a computational graph. 
\WRMLPgraphon*
Throughout the proof, we ensure that the structure of a partition underlying a computational kernel is preserved under the application of the weak regularity lemma.

\begin{proof}
Let $\epsilon>0$, and let $K$ be a computational kernel based on the partition $\mathcal{Q}$. 
First, apply \cref{lem:gs-reg-lem-refine} to the computational kernel $K$ with accuracy
$\epsilon/4$.
This yields a step kernel $K_k$, and a measurable partition $\cP_k$ of $[0,1]$, 
such that
\[
k = 2^{2\lceil 16/\epsilon^2\rceil},
\qquad
\|K-K_k\|_{\square} \le \epsilon/4.
\]
The partition $\cP_k$ is a general measurable partition (not necessarily an
interval partition, and not respecting the structure of a computational kernel). 

Let $\cL_r$ be the partition of $[0,1]$, with $r=L+d_0+d_L$ which is the union of:
\begin{itemize}
    \item all intermediate layers $\{U^{(1)},\dots,U^{(L-1)}\}$,
    \item the bias layer $\{U^{(\mathrm{bias})}\}$,
    \item the coarse input partition $\mathcal{C}^{\mathrm{in}}_{d_0}$, 
    \item the coarse output partition $\mathcal{C}^{\mathrm{out}}_{d_L}$. 
\end{itemize}
We call $\mathcal{L}_r$ the \emph{layer partition}. Let $\cP$ be the coarsest common refinement of $\cP_k$ and $\cL_r$.
We have 
\[
|\cP| \le |\cP_k|\,|\cL_r|=2^{2\lceil 16/\epsilon^2\rceil}(L+d_0+d_L),
\]
and $\cP$ is a measurable (not interval) partition.

By \cref{lem:attain}, the cut norm
$\|K_k-K_{\cP}\|_{\square}$ is attained on sets of the form
\[
S=\bigcup_{P_i\in s} P_i,
\qquad
T=\bigcup_{P_j\in t} P_j,
\]
for some  $s,t\subset\mathcal{P}$.
Using the definition of the projected kernel,
\begin{align*}
\|K_k-K_{\cP}\|_{\square}
&= \left| \int_S\int_T (K_k-K_{\cP})(x,y)\,dx\,dy \right| \\
&= \left| \int_S\int_T (K_k-K)(x,y)\,dx\,dy \right| \\
&\le \|K_k-K\|_{\square}.
\end{align*}
Therefore,
\[
\|K-K_{\cP}\|_{\square}
\le
\|K-K_k\|_{\square}
+
\|K_k-K_{\cP}\|_{\square}
\le \epsilon/2.
\]

Note that the parts of $\mathcal{P}$ that intersect $U^{(0)}$ constitute a refinement of $\mathcal{C}_{d_0}^{\mathrm{in}}$.  Similarly, the parts of $\mathcal{P}$ that intersect  $U^{(L)}$ constitute a refinement of $\mathcal{C}_{d_L}^{\mathrm{out}}$.
Note as well that $K$ is constant on every set of the form $C\times D$ or $D\times C$ where $C$ is a part in $\mathcal{C}_{d_0}^{\mathrm{in}}$ or $\mathcal{C}_{d_L}^{\mathrm{out}}$ and $D$ is any part of $\mathcal{Q}$ (the partition underlying $K$) that intersect (or equivalently, is a subset of) $U^{(1)}\cup\ldots\cup U^{(L-1)}\cup U^{(\mathrm{bias})}$.  

Hence, by the definition of projection and since $\mathcal{P}$ refines $\mathcal{L}_r$, the kernel  $K_{\mathcal{P}}$ has the same property: it is constant on every set of the form $C\times D$ or $D\times C$ where $C$ is a part in $\mathcal{C}_{d_0}^{\mathrm{in}}$ or $\mathcal{C}_{d_L}^{\mathrm{out}}$ and $D$ is any part of $\mathcal{P}$ that intersect $U^{(1)}\cup\ldots\cup U^{(L-1)}$.
We define the partition $\mathcal{P}'$ as the partition containing as parts the intervals of $\mathcal{C}_{d_0}^{\mathrm{in}}$ and $\mathcal{C}_{d_L}^{\mathrm{out}}$ and all parts of $\mathcal{P}$ outside $\mathcal{C}_{d_0}^{\mathrm{in}}$ and $\mathcal{C}_{d_L}^{\mathrm{out}}$. By the above analysis, $K_{\mathcal{P}}$ is a step graphon with respect to $\mathcal{P}'$.

For each  $\ell=1,\ldots,L-1$, let $\mathcal{P}^{(\ell)}$ be the set of parts of $\mathcal{P}'$ that intersect  (or equivalently are subsets of) $U^{(\ell)}$. Similarly define $\mathcal{P}^{(\mathrm{bias})}$. Note that $\mathcal{P}^{(\ell)}$ is a partition of $U^{(\ell)}$ and $\mathcal{P}^{(\mathrm{bias})}$ a partition of $U^{(\mathrm{bias})}$. Note as well that \[
|\mathcal{P}^{(\mathrm{bias})}|,|\mathcal{P}^{(\ell)}| \le 2^{2\lceil 16/\epsilon^2\rceil}(L+d_0+d_L).
\]

Let $M\in\mathbb{N}$ be the lowest common multiple of $d_0$ and $d_L$. Let $n=C(L+2)M$ for some $C\in\mathbb{N}$ to be specified later. 
Now, apply the equatizing partition lemma (Lemma \ref{l:equipartition}), to equatize   $\mathcal{P}^{(\ell)}$ into $CM$ parts $\mathcal{E}^{(\ell)}_{CM}$, each of measure $1/n$. Similarly, define $\mathcal{E}^{(\mathrm{bias})}_{CM}$. Here, we restrict the choice of $C$ to satisfy 
\[CM>\max_{\ell=1,\ldots,L-1,{(\mathrm{bias})}}|\mathcal{P}^{(\ell)}|.\] 
By the equatizing lemma, there exist numbers $h^{(\ell)}\leq |\mathcal{P}^{(\ell)}|$ such that the first $CM-h^{(\ell)}$ parts  of $\mathcal{E}^{(\ell)}_{CM}$ are the refinement parts $\mathcal{T}_{CM}^{(\ell)}$, and the last $h^{(\ell)}$ parts  of $\mathcal{E}_{CM}^{(\ell)}$ are the remainder parts $\mathcal{R}^{(\ell)}_{CM}$. Similarly, we consider $h^{(\mathrm{bias})}<|\mathcal{P}^{(\mathrm{bias})}|$, $\mathcal{T}_{CM}^{(\mathrm{bias})}$ and $\mathcal{R}_{CM}^{(\mathrm{bias})}$.

Denote $\mathcal{E}_{j}=\cup_{\ell}\mathcal{E}^{(\ell)}_{CM}\cup \mathcal{E}^{(\mathrm{bias})}_{CM}\cup \mathcal{C}_{d_0}^{\mathrm{in}} \cup \mathcal{C}_{d_L}^{\mathrm{out}}$, where $j=|\mathcal{E}_{j}| = CML+d_0+d_L$. Denote $\mathcal{R}=\cup_{\ell}\mathcal{R}_{CM}^{(\ell)}\cup\mathcal{R}_{CM}^{(\mathrm{bias})}$ and $R=\cup\mathcal{R}$.
Choose the refinement parts $E\subseteq P$, $P\in\mathcal{P}'$, so that
\[
\mu(E\cap Q)
=\frac{\mu(E)}{\mu(P)}\,\mu(P\cap Q)
\qquad\text{for every }Q\in\mathcal{Q}.
\]
This is possible by subdividing each $P\cap Q$ proportionally, using the nonatomicity of Lebesgue measure. Since $K$ is constant on the blocks of $\mathcal{Q}$, its average over a product of refinement parts equals its average over the corresponding parent block.

By construction, the projection $K_{\mathcal{E}_j}$ of $K$ upon $\mathcal{E}_j$ is equal to $K_{\mathcal{P}}$ for every $(x,y)$ outside the set $([0,1]\times R)\cup (R\times[0,1])$. Note that 
\begin{align*}
  \mu(([0,1]\times R)\cup (R\times[0,1])) & \leq 2 \big(h^{(\mathrm{bias})}+\sum_{\ell} h^{(\ell)}\big)/n \leq 2L|\mathcal{P}'|/n.
\end{align*}
Moreover, the pointwise distance between $K_{\mathcal{P}}$ and $K_{\mathcal{E}_j}$  for any $(x,y)\in ([0,1]\times R)\cup (R\times[0,1])$ is bounded by 2. Hence,
\[\|K_{\mathcal{P}}-K_{\mathcal{E}_j}\|_{\square} \leq \|K_{\mathcal{P}}-K_{\mathcal{E}_j}\|_{1} \leq 4L|\mathcal{P}'|/n.\]

Hence, by the triangle inequality
\[\|K-K_{\mathcal{E}_j}\|_{\square} \leq \epsilon/2 + 4L|\mathcal{P}'|/n.\]
We retroactively choose $n$ such that $4L|\mathcal{P}'|/n\leq\epsilon/2$. For this, we choose
\[n = 8ML(L+2)\left\lceil \frac{2^{2\lceil 16/\epsilon^2\rceil}}{\epsilon} \right\rceil,\]
which leads to 
\[4L|\mathcal{P}'|/n\le 4L2^{2\lceil 16/\epsilon^2\rceil}(L+d_0+d_L)/n \leq   4L2^{2\lceil 16/\epsilon^2\rceil}(L+2M)/n \]
\[\leq \frac{(L+2M)\epsilon}{2(ML+2M)}\leq \frac{\epsilon}{2},\]
while still $n$ being of the form $C(L+2)M$ for some $C\in\mathbb{N}$

Next, we map the partition $\mathcal{E}_j$ into an interval partition which underlies a computational kernel, i.e., which refines $\mathcal{L}_r$.
For this, using Lemma \ref{lem:bijection2interval}, for each $\ell\in\{1,\ldots, L-1,{(\mathrm{bias})}\}$ apply a measure preserving bijection $\phi^{(\ell)}$ mapping $\mathcal{E}^{(\ell)}_{CM}$ into the interval equipartition $\mathcal{I}^{(\ell)}_{CM}$ of $U^{(\ell)}$. 
Denote 
\[ \mathcal{J}=\mathcal{C}_{d_0}^{\mathrm{in}}\cup \mathcal{I}^{(1)}_{CM}\cup \ldots \cup \mathcal{I}^{(L-1)}_{CM}\cup \mathcal{I}^{(\mathrm{bias})}_{CM}\cup \mathcal{C}_{d_L}^{\mathrm{out}},\]
and note that $\mathcal{I}_n$ refines $\mathcal{J}$, which refines $\mathcal{L}_r$. 
Consider the corresponding measure preserving bijection $\phi$ over $[0,1]$, defined by
\[\phi(x)=\begin{cases}
   \phi^{(\ell)}(x), &  x\in U^{(\ell)},\ \text{for some}\ \ell=1,\ldots,L-1,{(\mathrm{bias})}\\
   x & \text{otherwise}.
\end{cases}\]
By Conditions 2--4, $K$ and $K_{\mathcal{E}_j}$ are unchanged by measure-preserving relabelings within $U^{(\mathrm{bias})}$. We therefore replace $\phi$ on $U^{(\mathrm{bias})}$ by the identity without changing $K^\phi$ or $(K_{\mathcal{E}_j})^\phi$. For this modified map, $\phi\in\mathcal{S}_L$.
Consider the kernel $K^{\phi}(x,y)=K(\phi(x),\phi(y))$, and note that $(K^{\phi})_{\mathcal{J}}$ is a computational kernel in $\mathcal{CK}(B,L,d_0,d_L)$ with hidden dimension
\[d=8ML\left\lceil \frac{2^{2\lceil 16/\epsilon^2\rceil}}{\epsilon} \right\rceil.\]
By the above construction, this kernel satisfies the commutation property between the measure-preserving bijection and the projection
\[(K^{\phi})_{\mathcal{J}} = (K_{\mathcal{E}_j})^{\phi}.
\]
Since
\[
h^{(\mathrm{bias})}+\sum_{\ell=1}^{L-1}h^{(\ell)}
\leq |\mathcal{P}'|-d_0-d_L<L|\mathcal{P}'|,
\]
the preceding error estimate is strict, and hence
\[\delta^{\mathrm{comp}}_{\square}\bigl(K,(K^{\phi})_{\mathcal{J}}\bigr)<\epsilon.\]
\end{proof}
\paragraph{Network Architecture Approximation.} We now use \cref{th:lipwithconstcomp} and \cref{thm:wrlcompgraphon} to prove \cref{thm:arch_approx}. We first establish a computational kernel approximation via \cref{thm:wrlcompgraphon}. Lastly, we use the Lipschitz continuity of graph network of \cref{th:lipwithconstcomp}, to bound the output.
\ArchApproxTwo*
\begin{proof}
Let $\mathbf{x} \in [0,1]^{d_0}$ be any input vector and let \(f_\mathbf{x}\in\mathcal{CS}(L,d_0,d_L)\) be its induced input signal given the parameters $L,d_0$, and $d_L$. Additionally, let $K\in\mathcal{CK}(B,L,d_0,d_L)$ denote the computational kernel of $\Theta$.

Let \(U^{(L)}\) be the output layer of $K$ (of length $1/(L+2)$) and $\mathcal{C}^{\mathrm{out}}_{d_L}$ be the coarse output partition of $K$, which partitions the interval $U^{(L)}$ into $d_L$ intervals of length $\frac{1}{(L+2)d_L}$, i.e. the Lebesgue measure $\mu$ of each of the intervals is $\frac{1}{(L+2)d_L}$.

By \cref{thm:wrlcompgraphon}, for any $\epsilon>0$, there exists a computational kernel $K'$ with hidden dimension
 \begin{align}
 \label{eq:inProofd}
&d=8ML\left\lceil \frac{2^{2\lceil (16(L+2)^2d_L^2(2B)^{2L})/\epsilon^2\rceil}(L+2)d_L(2B)^L}{\epsilon} \right\rceil.
 \end{align}
 approximating $K$ in computational cut distance with error
\[
\delta_{\square}^{\mathrm{comp}}(K, K') < \frac{\epsilon}{(L+2)d_L(2B)^L}.
\]
By Remark \ref{rem:inducedGeneral} (for any computational kernel there is a dense neural network that induces it) there exists a dense network $\Theta'\in\mathcal{NN}_{\mathrm{dense}}(B,L,d_0,d_L)$, with hidden dimension given by (\ref{eq:inProofd}), that induces $K'$.

Let $j\in[d_L]$ be any output channel and denote the $j$th interval of $\mathcal{C}^{\mathrm{out}}_{d_L}$ by $I^{\mathrm{out}}_j$. By \cref{lem:ker_equi}, we have
\[
\Theta
(\mathbf{x})_j = \Phi_{B,L}\big(K
,f_{\mathbf{x}}\big)\big(v\big), \quad \Theta'
(\mathbf{x})_j = \Phi_{B,L}\big(K'
,f_{\mathbf{x}}\big)\big(v\big),
\]
whenever $v\in I^{\mathrm{out}}_j$. Thus,
\begin{align*}
\big|&\Theta(\mathbf{x})_j - \Theta'(\mathbf{x})_j\big| 
= \big|\Phi_{B,L}(K, f_{\mathbf{x}})(v) - \Phi_{B,L}(K', f_{\mathbf{x}})(v)\big| 
\end{align*}
whenever $v\in I^{\mathrm{out}}_j$. Therefore, for any $v\in I^{\mathrm{out}}_j$:
\begin{align*}
\big|\Phi_{B,L}(K, f_{\mathbf{x}})(v) - \Phi_{B,L}(K', f_{\mathbf{x}})(v)\big|&=\frac{1}{\mu\left(I^{\mathrm{out}}_j\right)}\abs{\mu\left(I^{\mathrm{out}}_j\right) \left(\Phi_{B,L}(K, f_{\mathbf{x}})(v)-\Phi_{B,L}(K', f_{\mathbf{x}})(v)\right)}
\\&
=(L+2)d_L \abs{ \int_{I^{\mathrm{out}}_j}\left(\Phi_{B,L}(K, f_{\mathbf{x}})(u)-\Phi_{B,L}(K', f_{\mathbf{x}})(u)\right)du}
\\&\le(L+2)d_L\norm{ {\rm Out}_{K,f_{\mathbf{x}}}-{\rm Out}_{K',f_{\mathbf{x}}}}_{\square},
\end{align*}
where
 \[{\rm Out}_{K,f_{\mathbf{x}}}:=\mathbbm{1}_{U^{(L)}}\Phi_{B,L}(K,f_{\mathbf{x}}), \quad {\rm Out}_{K',f_{\mathbf{x}}}:=\mathbbm{1}_{U^{(L)}}\Phi_{B,L}(K',f_{\mathbf{x}}).\]
Lastly, by the Lipschitz property of the $B$-IR-MPNN (\cref{th:lipwithconstcomp})
\begin{align*}
(L+2)d_L\norm{ {\rm Out}_{K,f_{\mathbf{x}}}-{\rm Out}_{K',f_{\mathbf{x}}}}_{\square} \le(L+2)d_L (2B)^L \, \delta_{\square}^{\mathrm{comp}}(K, K') < \epsilon.
\end{align*}
Hence,
\begin{align*}
\big|&\Theta({\mathbf{x}})_j - \Theta'({\mathbf{x}})_j\big| < \epsilon.
\end{align*}
\end{proof}
\section{Implications for Approximation}
In this section, we prove our main results, providing explicit conditions under which dense deep ReLU networks fail to be universal. In other words, for sufficiently large input dimension $d_0$, \emph{any dense ReLU network architecture is not dense in ${\mathrm{Lip}}(d_0,d_L)$} (the space of Lipschitz continuous functions with Lipschitz constant bounded by $1$). First, we prove a variant of \cite{Yarotsky2017approximations}, Theorem 4., part (a), which depends only on one implicit parameter (in contrast to two in the original theorem). \cite{Yarotsky2017approximations}, Theorem 4. provides lower bounds on the number of computational units a network architecture must have in order to be able to approximate any function in ${\mathrm{Lip}}(d_0,d_L)$. Then, we use these two results to show that for a large enough input dimension, no ReLU architecture can approximate all functions with a small $\epsilon$ error. This highlights an inherent limitation of dense networks. 
\subsection{Lower Bounds Based on VC-Dimension}
We establish the following theorem  by following the steps of the proof of \cite{Yarotsky2017approximations}, Theorem 4. \cref{thm:VCbounds} gives a result close to the first part of \cite{Yarotsky2017approximations}, Theorem 4, but with explicit constants restricted to Lipschitz continuous functions.

Recall that for a hypothesis class $\mathcal{H}$ of Boolean functions on $[0,1]^d$,
the \emph{VC-dimension} is the size of the largest subset
$S \subset [0,1]^d$ that can be shattered by $\mathcal{H}$,
i.e., on which $\mathcal{H}$ can realize all dichotomies
(see \cref{sec:background}). 

Of particular interest is the case where $\mathcal{H}$ consists of Boolean
functions obtained by thresholding the scalar output of a ReLU network with
fixed architecture and variable weights and biases. Concretely, let
$\Theta_{(\weights,\bias)}: [0,1]^d \to \mathbb{R}$ denote the function computed by such a
network with parameters ${(\weights,\bias)}$, and consider the class
\[
\mathcal{H} = \bigl\{\, x \mapsto \mathds{1}_{\{\Theta_{(\weights,\bias)}(x) > a\}} \;\big|\;(\weights,\bias) \hspace{0.1cm}{\rm network}\hspace{0.1cm}{\rm parameters}\hspace{0.1cm},\ a \in \mathbb{R} \bigr\}.
\]
Here, the parameter $a \in \mathbb{R}$ represents a threshold applied to the
network output. When the network parameters are unconstrained, this threshold
can be fixed to $a=0$ without loss of generality, since it can be absorbed into
the bias of the last layer. We emphasize that no margin assumption is
imposed here; the thresholding is purely combinatorial and used only to define
a Boolean-valued hypothesis class.

In this setting, \cite{Bartlett2009}, Theorem 8.7 shows that the VC-dimension satisfies 
\begin{align}
\mathrm{VCdim}(\mathcal{H}) \leq c \cdot W^2,\label{eq:VCdimLW}
\end{align}
where $W$ is the total number of parameters (weights and biases) in the network and $c$ is a universal constant independent of the architecture. Recall that ${\mathrm{Lip}}(d_0,d_L)$ denotes the space of $1$-Lipschitz continuous functions from $[0,1]^{d_0}$ to $\mathbb{R}^{d_L}$ bounded by $1$. Theorem 3 of \cite{tightVC2017} gives a depth-dependent upper bound of order $O(WL\log W)$,
which improves the quadratic bound asymptotically for fixed depth.
\Yarotsky*
\begin{proof}
We first reduce to the case of a scalar output. Since the Lipschitz constant and the approximation error are defined using the infinity norm over the output, every output channel can be treated independently. Hence, it suffices to prove the result for $d_L = 1$.

Given a positive real number $N$ to be chosen later, choose $S$ as a set of $\lceil N\rceil^{d_0}$ points $\mathbf{x}_1, \dots, \mathbf{x}_{\lceil N\rceil^{d_0}}$ in the cube $[0,1]^{d_0}$ such that the $\ell_\infty$ distance between any pair of them is not less than $\tfrac{1}{N}$. Given any assignment of values $y_1, \dots, y_{\lceil N\rceil^{d_0}} \in \mathbb{R}$, we can construct a Lipschitz continuous function $f$ satisfying $f(\mathbf{x}_m) = y_m$ for all $m$ by setting
\begin{equation}
    f(\mathbf{x}) = \sum_{m=1}^{\lceil N\rceil^{d_0}} y_m \phi\bigl(N(\mathbf{x} - \mathbf{x}_m)\bigr),
    \label{eq:fDef}
\end{equation}
with $\phi : \mathbb{R}^{d_0} \to \mathbb{R}$ defined by
\[
\phi(\mathbf{x})=\max\{1-2\|\mathbf{x}\|_\infty,0\}.
\]
The separation of the centers implies
\[
\|f\|_\infty\leq\max_m|y_m|,
\]
and, for all $\mathbf{x},\mathbf{y}\in[0,1]^{d_0}$,
\begin{equation}
|f(\mathbf{x})-f(\mathbf{y})|
\leq 2N\max_m|y_m|\,\|\mathbf{x}-\mathbf{y}\|_\infty.
\label{eq:DnfBound}
\end{equation}
Therefore, if $\max_m|y_m|\leq 1/(2N)$, then $f$ is $1$-Lipschitz with respect to the infinity norm.
Now, set
\begin{equation}
    N = \frac{1}{6\epsilon}.
    \label{eq:epsDef}
\end{equation}
Then $\max_\mathbf{x}|f(\mathbf{x})|\le3\epsilon$. In particular, when $\epsilon\le 1/3$, $\max_\mathbf{x}|f(\mathbf{x})|\le1$. Therefore, $f \in {\mathrm{Lip}}(d_0,d_L)$. 
Suppose that there is a ReLU network architecture that can approximate, by adjusting its weights, any $f \in {\mathrm{Lip}}(d_0,d_L)$ with error at most $\epsilon$. By $\Theta_{(\weights,\bias)}(\mathbf{x})$, we denote the output of the network with parameters $(\weights,\bias)$.   

Consider any assignment $\mathbf{z}$ of Boolean values $z_1, \ldots, z_{\lceil N\rceil^{d_0}} \in \{0,1\}$. 
Set
\[
    y_m = \frac{z_m}{2N}, \quad m = 1, \ldots, \lceil N\rceil^{d_0},
\]
and let $f$ be given by \cref{eq:fDef}; then \cref{eq:DnfBound} holds and hence $f \in {\mathrm{Lip}}(d_0,d_L)$. 
By assumption, there exist fixed parameters $(\weights_z,\bias_z)$, such that for all $m$ we have 
$$|\Theta_{(\weights_z,\bias_z)}(\mathbf{x}_m) - y_m| \leq \epsilon,$$ and in particular
\[
\begin{aligned}
\Theta_{(\weights_z,\bias_z)}(\mathbf{x}_m)
&\geq \frac{1}{2N} - \epsilon > \frac{1}{4N},
&&\text{if } z_m = 1, \\
\Theta_{(\weights_z,\bias_z)}(\mathbf{x}_m)
&\leq \epsilon < \frac{1}{4N},
&&\text{if } z_m = 0.
\end{aligned}
\]
so the thresholded network ${\hat{\Theta}}_{(\weights_z,\bias_z)} = \mathbf{1}(\Theta_{(\weights_z,\bias_z)} > \frac{1}{4N})$ has outputs
\[
{\hat{\Theta}}_{(\weights_z,\bias_z)}(\mathbf{x}_m) = z_m, \quad m = 1, \ldots, \lceil N\rceil^{d_0}.
\]
Since the Boolean values $z_m$ were arbitrary, we conclude that the subset $S$ is shattered and hence
\[
\mathrm{VCdim}(\mathcal{H}) \geq \lceil N\rceil^{d_0} \geq N^{d_0}.
\]
Expressing $N$ through $\epsilon$ with \cref{eq:epsDef}, we obtain
\begin{equation}
\mathrm{VCdim}(\mathcal{H}) \geq \left(6 \epsilon\right)^{-{d_0}}.
\label{eq:vclower1}
\end{equation}
To establish the Theorem, we apply the inequality in \cref{eq:VCdimLW} to the class $\mathcal{H}$:
\begin{equation}
\mathrm{VCdim}(\mathcal{H}) \leq c W^2,
\label{eq:vcupper1}
\end{equation}
where $W$ is the total number of parameters (weights and biases) in this architecture. Combining \cref{eq:vclower1} with \cref{eq:vcupper1}, we obtain the desired lower bound
\[
W \geq c^{-1/2} \, (6\epsilon)^{-{d_0}/2}.
\]
\end{proof}

Next, we relate the number of parameters of a neural network in $\mathcal{NN}_{\mathrm{dense}}(B,L,d_0,d_L)$ to the hidden and input dimensions, using formulae compatible with Corollary \ref{thm:arch_approx} and Theorem \ref{thm:VCbounds} with the choice of $\epsilon=1/8$. This correspondence will be used to derive the expressivity bound on strongly dense neural networks.

\begin{lemma}\label{lem:function_family}
Let $\epsilon_0=1/8$. Let $B\ge1$, $L,d_L\in\mathbb{N}$ with $L\ge2$, and let $c$ be the constant from \cref{thm:VCbounds}. Let $d_0$ satisfy
\begin{equation}\label{eq:exp_bound}
d_0 \geq 17\log_2\!\bigl(
c^{1/2}L^3(L+2)^2d_L^4(2B)^{2L}
\bigr)
+ 17\cdot 2^{14}(L+2)^2d_L^2(2B)^{2L} + 306.
\end{equation}
Consider the family $\mathcal{H}_{\epsilon_0}$ of ReLU networks with strictly less than $c^{-1/2}(6\epsilon_0)^{-d_0/2}$ parameters. Then, the space of dense neural networks in $\mathcal{NN}_{\mathrm{dense}}(B,L,d_0,d_L)$ with hidden dimension $$d=8ML\left\lceil \frac{2^{2\lceil (16(L+2)^2d_L^2(2B)^{2L})/(\epsilon_0/2)^2\rceil}(L+2)d_L(2B)^L}{(\epsilon_0/2)} \right\rceil$$
is a subset of $\mathcal{H}_{\epsilon_0}$.    
\end{lemma}
\begin{proof}
We now show that any network
$\Theta\in\mathcal{NN}_{\mathrm{dense}}(B,L,d_0,d_L)$ with hidden dimension $$d=8ML\left\lceil \frac{2^{2\lceil (16(L+2)^2d_L^2(2B)^{2L})/(\epsilon_0/2)^2\rceil}(L+2)d_L(2B)^L}{(\epsilon_0/2)} \right\rceil$$ has fewer than
$c^{-1/2}(6\epsilon_0)^{-d_0/2}$ parameters. 
By substituting $\epsilon_0 = 1/8$ and collecting powers of $2$, we obtain
\begin{align*}
\left\lceil
\frac{2^{2\lceil (16(L+2)^2d_L^2(2B)^{2L})/(\epsilon_0/2)^2\rceil}
(L+2)d_L(2B)^L}{(\epsilon_0/2)}\right\rceil= \left\lceil
2^{2\lceil 2^{12}(L+2)^2d_L^2(2B)^{2L}\rceil+4}
(L+2)d_L(2B)^L
\right\rceil \\[0.3em]
\end{align*}
Notice that 
\begin{align*}
\left\lceil
2^{2\lceil 2^{12}(L+2)^2d_L^2(2B)^{2L}\rceil+4}
(L+2)d_L(2B)^L
\right\rceil&\le\left(
2^{ 2(2^{12}(L+2)^2d_L^2(2B)^{2L}+1)+4}
(L+2)d_L(2B)^L+1\right)\\
&=\left(
2^{ 2^{13}(L+2)^2d_L^2(2B)^{2L}+6}
(L+2)d_L(2B)^L+1\right)  
\end{align*}
Therefore,
\begin{align*}
d&\le 8ML\cdot
2^{ 2^{13}(L+2)^2d_L^2(2B)^{2L}+6}
(L+2)d_L(2B)^{L}+8ML
\\&= 2^3ML\cdot
2^{ 2^{13}(L+2)^2d_L^2(2B)^{2L}+6}
(L+2)d_L(2B)^{L}+2^3ML
\\&= 
(ML(L+2)d_L(2B)^L)
2^{2^{13}(L+2)^2d_L^2(2B)^{2L}+9}
+ 2^3ML.
\end{align*}
Since $M\le d_0 d_L$, we have
\begin{align*}
d&\le 
(d_0d_L^2L(L+2)(2B)^L)
2^{2^{13}(L+2)^2d_L^2(2B)^{2L}+9}
+ 2^3d_0d_LL.
\end{align*}
We now calculate $W$, the number of parameters (including both the weights and the biases) of $\Theta$. We have
\begin{align*}
W
&= (L-2)d^2 + (d_0 + d_L + L - 1)d + d_L,
\end{align*}
since there are $(L-2)d^2+(d_0 + d_L)d$ weights and $(L - 1)d + d_L$ biases in a fixed-width network. 

Since $d\ge d_L$, we have
\begin{align*}
(d_0 + d_L + L - 1)d + d_L \le (d_0 + d_L + L)d.
\end{align*}
Since $L\ge 2$, we have
\begin{align*}
(d_0 + d_L + L)d\le d_0d_LLd.
\end{align*}
Thus
\begin{align*}
W=(L-2)d^2 + (d_0 + d_L + L - 1)d + d_L\le (L-1)d^2 + d_0d_LLd.
\end{align*}
The bound $d\le(d_0d^2_LL(L+2)(2B)^L)
2^{2^{13}(L+2)^2d_L^2(2B)^{2L}+9}
+ 2^3d_0d_LL$ yields
\begin{align*}
W&\le (L-1)
\left(\left(d_0d^2_LL(L+2)(2B)^L\right)2^{2^{13}(L+2)^2d_L^2(2B)^{2L}+9}+  2^3d_0d_LL\right)^2
\\&\quad + d_0d_LL\left(\left(d_0d^2_LL(L+2)(2B)^L\right)
2^{2^{13}(L+2)^2d_L^2(2B)^{2L}+9}
+ 2^3d_0d_LL\right)
\\&
=(L-1)
\left(d_0d_L^2L(L+2)(2B)^L\right)^2\left(2^{2^{13}(L+2)^2d_L^2(2B)^{2L}+9}\right)^2
\\&\quad+(L-1)2^4d_0d_LL\left(d_0d_L^2L(L+2)(2B)^L\right)
2^{2^{13}(L+2)^2d_L^2(2B)^{2L}+9}
\\&\quad+(L-1)\left(2^3d_0d_LL\right)^2
\\&\quad + d_0d_LL\left((d_0d^2_LL(L+2)(2B)^L)
2^{2^{13}(L+2)^2d_L^2(2B)^{2L}+9}\right)
\\&\quad+ d_0d_LL\left(2^3d_0d_LL\right)
\\&
= (L-1)d_0^2d_L^4L^2(L+2)^2(2B)^{2L}
2^{2^{14}(L+2)^2d_L^2(2B)^{2L}+18} \\
&\quad + (L-1)d^2_0d^3_LL^2(L+2)(2B)^L
2^{2^{13}(L+2)^2d_L^2(2B)^{2L}+13} 
\\&\quad + (L-1)2^6d_0^2d_L^2L^2
\\
&\quad  
+ d_0^2d_L^3L^2(L+2)(2B)^L
2^{2^{13}(L+2)^2d_L^2(2B)^{2L}+9}
\\&\quad+ 2^3d_0^2d_L^2L^2 =:(*)
\end{align*}
Notice that
\begin{align*}
&(L-1)2^6d_0^2d_L^2L^2  
+ d_0^2d_L^3L^2(L+2)(2B)^L
2^{2^{13}(L+2)^2d_L^2(2B)^{2L}+9}
+ 2^3d_0^2d_L^2L^2 
\\&\qquad\qquad\qquad\le 2\cdot d_0^2d_L^3L^2(L+2)d_L(2B)^L
2^{2^{13}(L+2)^2d_L^2(2B)^{2L}+9}
\\&\qquad\qquad\qquad= d_0^2d_L^3L^2(L+2)d_L(2B)^L
2^{2^{13}(L+2)^2d_L^2(2B)^{2L}+10}
\\&\qquad\qquad\qquad\le d_0^2d_L^3L^2(L+2)d_L(2B)^L
2^{2^{13}(L+2)^2d_L^2(2B)^{2L}+12}
\end{align*}
Thus
\begin{align*}
& (L-1)d^2_0d^3_LL^2(L+2)(2B)^L
2^{2^{13}(L+2)^2d_L^2(2B)^{2L}+13} 
\\&\quad + (L-1)2^6d_0^2d_L^2L^2
\\
&\quad  
+ d_0^2d_L^3L^2(L+2)(2B)^L
2^{2^{13}(L+2)^2d_L^2(2B)^{2L}+9}
\\&\quad+ 2^3d_0^2d_L^2L^2
\\&< (L+2)d_0^2d_L^4L^2(L+2)(2B)^{2L}
2^{2^{14}(L+2)^2d_L^2(2B)^{2L}+18}
\\&=d_0^2d_L^4L^2(L+2)^2(2B)^{2L}
2^{2^{14}(L+2)^2d_L^2(2B)^{2L}+18}
\end{align*}
All in all, we get
\begin{align*}
(*)&< d_0^2d_L^4L^3(L+2)^2(2B)^{2L}
2^{2^{14}(L+2)^2d_L^2(2B)^{2L}+18}
\\&\le d_0^2d_L^4L^3(L+2)^2(2B)^{2L}
2^{2^{14}(L+2)^2d_L^2(2B)^{2L}+18}
:=\widetilde{W}
\end{align*}
That is, $W< \widetilde{W}.$
We seek a condition on $d_0$ such that
\[
\widetilde{W}
\le c^{-1/2}\left(\frac{8}{6}\right)^{d_0/2},
\]
which implies
$W < c^{-1/2}(6\epsilon_0)^{-d_0/2}$. Taking logarithms, this condition is equivalent to
\[
\log_2\!\bigl(c^{1/2}\widetilde{W}\bigr)
\le \frac{d_0}{2}\log_2\!\left(\frac{4}{3}\right).
\]
Substituting the expression for $\widetilde{W}$ yields
\begin{align*}
\log_2\!\bigl(
c^{1/2}L^3(L+2)^2d_L^4(2B)^{2L}
\bigr)
+ 2^{14}(L+2)^2d_L^2(2B)^{2L} + 18
\le \frac{d_0}{2}\log_2\!\left(\frac{4}{3}\right) - 2\log_2(d_0).
\end{align*}

A direct calculation shows that for all $d_0 > 1800$,
\[
\frac{1}{17}d_0
\le \frac{d_0}{2}\log_2\!\left(\frac{4}{3}\right) - 2\log_2(d_0).
\]
Therefore, it suffices to require
\begin{align*}
17\log_2\!\bigl(
c^{1/2}L^3(L+2)^2d_L^4(2B)^{2L}
\bigr)
+ 17\cdot 2^{14}(L+2)^2d_L^2(2B)^{2L} + 306
\le d_0,
\end{align*}
which is precisely \cref{eq:exp_bound}.
\end{proof}
From the above theorem, we get the following.
\ExpressivityBound*
\begin{proof} 
Fix $\epsilon_0 := 1/8$. Consider the ReLU network architecture with depth $L$, input and output dimensions $d_0,d_L$, and hidden dimension $$d=8ML\left\lceil \frac{2^{2\lceil (16(L+2)^2d_L^2(2B)^{2L})/(\epsilon_0/2)^2\rceil}(L+2)d_L(2B)^L}{(\epsilon_0/2)} \right\rceil$$
By \cref{lem:function_family}, this architecture has strictly fewer than $c^{-1/2}(6\epsilon_0)^{-d_0/2}$ parameters. By \cref{thm:VCbounds}, there exists $f\in\mathrm{Lip}(d_0,d_L)$ that no network with this architecture can approximate with error at most $\epsilon_0$.

Suppose, for the sake of contradiction, that
$\mathcal{NN}_{\mathrm{dense}}(B,L,d_0,d_L)$ is a universal approximator
of ${\mathrm{Lip}}(d_0,d_L)$. Then there exists
$\Theta \in \mathcal{NN}_{\mathrm{dense}}(B,L,d_0,d_L)$ such that
\[
\|f - \Theta\|_\infty \le \epsilon_0/2.
\]

By \cref{thm:arch_approx}, any dense ReLU network
$\Theta \in \mathcal{NN}_{\mathrm{dense}}(B,L,d_0,d_L)$ admits an
approximating dense ReLU network
$\Theta' \in \mathcal{NN}_{\mathrm{dense}}(B,L,d_0,d_L)$ with controlled
hidden dimension
\begin{align*}
d
&= 8ML\left\lceil
\frac{2^{2\lceil (16(L+2)^2d_L^2(2B)^{2L})/(\epsilon_0/2)^2\rceil}
(L+2)d_L(2B)^L}{\epsilon_0/2}
\right\rceil 
\end{align*}
such that
\[
\|\Theta - \Theta'\|_\infty \le \epsilon_0/2.
\]
By the triangle inequality,
\[
\|f - \Theta'\|_\infty
\le \|f - \Theta\|_\infty + \|\Theta - \Theta'\|_\infty
\le \epsilon_0.
\]
This contradiction completes the proof.
\end{proof}
\section{Deep Neural Networks as Graph-Based Models}\label{app:graph_ect}
In this section, we define the computational graph of a feedforward network and formalize, in \cref{lem:graf_equi}, the connection between graph-based neural computations and ReLU networks. We then present \cref{lem:MPL-graph-graphon}, which is a direct consequence of the analysis in Appendix~E of \cite{signal23} and \cite{signal25}, connecting graph-based neural computations to kernel-based neural computations.

\emph{Graph Neural Networks} and specifically \textit{message passing networks (MPNNs)}, form a class of neural networks designed to process graph-structured data \citep{xu2018powerful, Rossi23} by iteratively updating node embeddings through the exchange of messages between nodes. Unlike general trainable MPNNs, we consider the following special predefined network.

\paragraph{SR-MPNN.} Let $L \in \mathbb{N}$; the $L$-layer \emph{sum-ReLU MPNN (SR-MPNN)} processes attributed graphs into output vertex attributes as follows. The SR-MPNN is applied on an attributed graph $(G,\f)$ with adjacency matrix  $\mathbf{A}=(A_{i,j})_{i,j\in[n]}$ as the function defined by $\Psi_L(G,\mathbf{f}) := \mathbf{f}^{(L)}$, where $\mathbf{f}^{(0)} := \mathbf{f}$ and for each layer $\ell = 1,\ldots,L-1$ the update rule is:
\begin{align}
\mathbf{f}^{(\ell)}(i) := {\mathrm{ReLU}}\!\left(\frac{1}{n}\sum_{j \in [n]} A_{i,j}\mathbf{f}^{(\ell-1)}(j)\right).
\label{eq:mpnns2}
\end{align}
The output is computed like in \cref{eq:mpnns2} with $\ell = L$, without ReLU.

\paragraph{The computational graph of a neural network.} To represent a feedforward neural network within the SR-MPNN framework, we associate it with a directed computational graph in which vertices may be repeated.
These repetitions are a technical but essential feature of the construction. 

Intuitively, when a feedforward network is implemented as an SR-MPNN on its computational graph, message passing uses \emph{normalized sum aggregation}.
If each neuron were represented by a single vertex, this normalization would excessively downscale signals as they propagate through the layers, so that the resulting SR-MPNN would no longer implements the intended feedforward computation.
To balance the contributions of different neurons and ensure that signal magnitudes are preserved across layers, we therefore repeat each input, bias, and output neuron several times; the precise reason for this repetition will become clear later in the analysis.

A \emph{partition} of the vertex set $V$ of a graph is a sequence of subsets $(V_j\subset V)_{j=1}^J$ such that $\cup_j V_j=V$ and $V_j\cap V_i=\emptyset$ for every $i\neq j$. The partition is called \emph{balanced} if $|V_i|=|V_j|$ for every $i,j$, where $|V_i|$ is the number of vertices in $V_i$.

We call a graph $G=(V,E)$, with an adjacency matrix $\mathbf{A}=(A_{v,u})_{v,u\in[n]}$, a depth-$L$ \emph{computational graph} with respect to the parameters $L,d_0,d_L\in\mathbb{N}$, if $G$ satisfies the following four conditions.
\begin{itemize}[leftmargin=0.35cm]\itemsep-0.4em
\item \emph{Condition 1.}
The graph is a weighted graph. The size of the graph $n$ is divisible by $M(L+2)$, where $M$ is the least common multiple of $d_0$ and $d_L$. 
\item[ ]
\emph{Notation:} The vertex set $V$ is endowed with a balanced partition into $L+2$ sets of size $d:=n/(L+2)$: \[
V = V^{(0)} \cup V^{(1)} \cup \cdots \cup V^{(L)} \cup V^{(\mathrm{bias})},
\]
where each \(V^{(\ell)}\), $\ell=0,\ldots,L$, is called \emph{layers}, and $V^{(\mathrm{bias})}$ is called the \emph{bias}. Specifically, we call \(V^{(0)}\) the  \emph{input layer} and \(V^{(L)}\) the \emph{output layer}. The layers \(V^{(0)}\) and \(V^{(L)}\) are endowed with balanced partitions 
\((V^{(\mathrm{in})}_i)^{d_0}_{i=1}\) and \((V^{(\mathrm{out})}_i)^{d_L}_{i=1}\) into \(d_0\) and \(d_L\) cells, respectively,
each part having size \(d/d_0\) and \(d/d_L\). We call \((V^{(\mathrm{in})}_i)^{d_0}_{i=1}\) \emph{input partition} and \((V^{(\mathrm{out})}_i)^{d_L}_{i=1}\) the \emph{output partition}.

\item \emph{Condition 2.}
 The graph satisfies $A_{v,u}=A_{v,u'}$ whenever $u$ and $u'$ belong to the same subset in \((V^{(\mathrm{in})}_i)^{d_0}_{i=1}\). Moreover,  $A_{v,u}=A_{v',u}$ whenever $v$ and $v'$ belong to the same part of \((V^{(\mathrm{out})}_i)^{d_L}_{i=1}\). In addition, $A_{v,u}=A_{v,u'}$ whenever $u$ and $u'$ belong to $V^{(\mathrm{bias})}$.

\item  \emph{Condition 3.}
The graph satisfies \(A_{v,u}=0\) whenever either
\(u\in V^{(\ell)}\) and \(v\notin V^{(\ell+1)}\) for some
\(\ell\in\{0,\dots,L-1\}\), or \(u\in V^{(\mathrm{bias})}\) and
\(v\in V^{(0)}\), or \(v\in V^{(\mathrm{bias})}\) and \(u\notin V^{(\mathrm{bias})}\), or \(u\in V^{(L)}\). 
\item  \emph{Condition 4.}
For any $v,u\in V^{(\mathrm{bias})}$, $A_{v,u}=(L+2)$.
\end{itemize}
We call $d_0$ the \emph{input dimension}, $d_L$ the \emph{output dimension}, and refer to $d$ as the \emph{hidden dimension} for $\ell=1,\dots,L-1$. By $\mathcal{CG}(L,d_0,d_L)$, we denote the collection of all depth-$L$ computational graphs 
 as defined above. 
\paragraph{Computational Node Feature Vectors.}
Given parameters \(L,d_0,d_L\), a \emph{computational node feature vector} is any node feature vector   which is constant on each part in $(V^{(\mathrm{in})}_i)_{i=1}^{d_0}$ and $(V^{(\mathrm{out})}_i)_{i=1}^{d_L}$, and constant 
on $V^{(\mathrm{bias})}$.
%
%
We define a \emph{computational input node feature
vector} as any computational node feature vector satisfying \(\mathbf{f}(v)=1\) for all \(v\in V^{(\mathrm{bias})}\) and \(\mathbf{f}(v)=0\) for all \(v\notin V^{(0)}\cup V^{(\mathrm{bias})}\). 
The values of \(\mathbf{f}\) on the sets of \((V^{(\mathrm{in})}_i)_{i=1}^{d_0}\) are interpreted as the input vector to the network, while the constant value on \(V^{(\mathrm{bias})}\) represents a bias signal that is present at every layer. \emph{Condition 4} of computational graph, i.e., \(A_{v,u}=(L+2)\) for all \(v,u\in V^{(\mathrm{bias})}\), ensures that this bias signal is unchanged during propagation, i.e., under successive SR-MPNN updates. 


\paragraph{Computational Graph Induced by Network Parameters.} Given a neural network, one can induce a computational graph on  which the SR-MPNN implements the forward propagation of the neural network. 
Let $(\weights,\bias)$ be the parameters of the neural network $\Theta_{(\weights,\bias)}\in\mathcal{NN}(L,d_0,d_L)$, with hidden dimension $d$ divisible by $d_0$ and $d_L$. Denote $n=d(L+2)$. We define the computational graph $G=G_{(\weights,\bias)}\in\mathcal{CG}(L,d_0,d_L)$ \emph{induced} by $\Theta_{(\weights,\bias)}$ as follows. For every $(v,u)\in[n]^2$:
\begin{itemize}[leftmargin=0.35cm]\itemsep-0.4em
    \item $A_{v,u}=W^{(1)}_{i,j}$ if $u$ is in \(V^{(\mathrm{in})}_j\) and $v$ is the $i$th vertex of $V^{(1)}$. 
    \item   
    $A_{v,u}=W^{(L)}_{i,j}$ if $u$ is the $j$th vertex of $V^{(L-1)}$ and $v$ is in \(V^{(\mathrm{out})}_i\). 
    \item   
    For $\ell=2,\ldots,L-1$, $A_{v,u}=W^{(\ell)}_{i,j}$ if $u$ is the $j$th vertex of $V^{(\ell-1)}$ and $v$ is the $i$th vertex of $V^{(\ell)}$. 
    \item For $\ell=1,\ldots,L-1$, $A_{v,u}=b^{(\ell)}_i$ if $u$ is in $V^{(\mathrm{bias})}$ and $v$ is the $i$th vertex of $V^{(\ell)}$. Moreover, $A_{v,u}=b^{(L)}_i$ if $u$ is in $V^{(\mathrm{bias})}$ and $v$ is in \(V^{(\mathrm{out})}_i\).
\end{itemize}
We encode an input vector $\mathbf{x} \in \mathbb{R}^{d_0}$ as node attributes \(\mathbf{f}_{\mathbf{x}} : V \to \mathbb{R}\)
defined by
\[
\mathbf{f}_{\mathbf{x}}(v) :=
\begin{cases}
x_i, & \text{if } v \text{ is in }V^{(\mathrm{in})}_i, \\
1, & \text{if } v\in V^{(\mathrm{bias})}, 
\\
0, & \text{otherwise}.
\end{cases}
\]

In the above construction, when writing ``the $j$th cell of the partition...'' we implicitly assume that the cells are always sorted in increasing order.
%

\paragraph{Networks as SR-MPNNs.} Lemma \ref{lem:graf_equi} shows that the output of a $B$-strongly dense neural network can be expressed in terms of the
$L$-layer SR-MPNN applied to its induced computational graph.  
\begin{restatable}[]{lemma} {MLPasgraphMPNN}~\label{lem:graf_equi}
Let $(\weights,\bias)$ be the parameters of $\Theta_{(\weights,\bias)}\in\mathcal{NN}(L,d_0,d_L)$. Then, for any input \(\mathbf{x} \in \mathbb{R}^{d_0}\) and output channel $i\in[d_L]$, we have
\[
\Theta_{(\mathbf{W},\mathbf{b})}
(\mathbf{x})_i = \Psi_{L}\big(G_{(\mathbf{W},\mathbf{b})}
,
\mathbf{f}_{\mathbf{x}}\big)\big(v\big),
\]
whenever $v$ is in \(V^{(\mathrm{out})}_i\).
\end{restatable}
The proof of \cref{lem:graf_equi} follows the same steps as the proof of \cref{lem:ker_equi}.

\paragraph{Computational Kernel as An Induced Computational Graph}  We call any depth-$L$ computational graph with respect to the parameters $L,d_0,d_L\in\mathbb{N}$ a
\emph{dense computational graph with bound \(B\)} if the graph is \([-B,B]\)-weighted for some \(B\geq L+2\). In this case, after relabeling the vertices to match the layer and input/output interval partitions, we can induce the computational graph into a computational kernel using the method introduced in \cref{sec:background}.

\cite{signal23,signal25} show that applying an MPNN to an attributed graph and then inducing an attributed kernel yields the same representation as first inducing the attributed kernel and then applying the MPNN (see e.g. Appendix E in \cite{signal25}). A direct result of this theorem in our setting is stated as follows.
\begin{restatable}[]{lemma}{graphgraphon}~\label{lem:MPL-graph-graphon}
Let $L \in \mathbb{N}$ and $B>0$.  
For any $[-B,B]$-weighted attributed graph $(G,\mathbf{f})$ with node set $[n]$, applying an $L$-layer $B$-IR-MPNN to its induced attributed kernel $(K_G,f_\f)$ corresponds to applying the standard $L$-layer ReLU graph network to $G$ in the following way: for every $i \in [n]$ and every $x$ in the interval $I_i$:
\[ 
\Psi_L(G, \mathbf{f})(i)=\Phi_{B,L}(K_G, f_\mathbf{f})(x).
\]
\end{restatable} 
\end{document}